\documentclass[journal]{IEEEtran}

\usepackage[T1]{fontenc}
\usepackage[utf8]{inputenc}
\usepackage{color}
\usepackage[brazilian,english]{babel}

\usepackage{amsmath}
\usepackage{amssymb}
\usepackage{amsfonts}

\usepackage{graphicx}
\usepackage{psfrag}
\usepackage{subcaption}

\usepackage{amsthm}
\usepackage{thmtools}
\usepackage{thm-restate}

\usepackage{tikz}
\usepackage{pgfplots}
\usepgfplotslibrary{patchplots}
\usetikzlibrary{fit,positioning}

\usepackage{algpseudocode}

\usepackage{xspace}
\usepackage{multirow}
\usepackage{mathtools}
\usepackage{cancel}

\newcommand{\N}{\mathcal{N}\xspace}
\newcommand{\M}{\mathcal{M}\xspace}
\newcommand{\R}{\mathbb{R}\xspace}
\newcommand{\E}{\mathbb{E}\xspace}

\DeclareMathOperator{\Cov}{Cov}

\theoremstyle{definition}
\newtheorem{definition}{Definition}

\usepackage{placeins}

\begin{document}

\title{Necessary and Sufficient Conditions for Surrogate Functions of Pareto
  Frontiers and Their Synthesis Using Gaussian Processes}
\author{Conrado S. Miranda,
Fernando J. Von Zuben,~\IEEEmembership{Senior Member,~IEEE}
\thanks{C. S. Miranda and F. J. Von Zuben are with the School of Electrical and
  Computer Engineering, University of Campinas (Unicamp), Brazil. E-mail:
contact@conradomiranda.com, vonzuben@dca.fee.unicamp.br}}

\maketitle

\begin{abstract}
  This paper introduces the necessary and sufficient conditions that surrogate
  functions must satisfy to properly define frontiers of non-dominated solutions
  in multi-objective optimization problems. These new conditions work directly
  on the objective space, thus being agnostic about how the solutions are
  evaluated. Therefore, real objectives or user-designed objectives' surrogates
  are allowed, opening the possibility of linking independent objective
  surrogates. To illustrate the practical consequences of adopting the proposed
  conditions, we use Gaussian processes as surrogates endowed with monotonicity
  soft constraints and with an adjustable degree of flexibility, and compare
  them to regular Gaussian processes and to a frontier surrogate method in the
  literature that is the closest to the method proposed in this paper. Results
  show that the necessary and sufficient conditions proposed here are finely
  managed by the constrained Gaussian process, guiding to high-quality
  surrogates capable of suitably synthesizing an approximation to the Pareto
  frontier in challenging instances of multi-objective optimization, while an
  existing approach that does not take the theory proposed in consideration
  defines surrogates which greatly violate the conditions to describe a valid
  frontier.
\end{abstract}


\IEEEpeerreviewmaketitle

\section{Introduction}
\IEEEPARstart{M}{ulti-objective} optimization (MOO), also called multiple
criteria optimization~\cite{gandibleux2006multiple}, is an extension of the
standard single-objective optimization, where the objectives may be
conflicting with each other~\cite{miettinen1999nonlinear,deb2001multi}. When a
conflict exists, we are no more looking for a single optimal solution but for a
set of solutions, each one providing a trade-off on the objectives and none
being better than the others. This solution set is called the Pareto set and
its counterpart in the objective space is denoted the Pareto frontier.

The Pareto frontier is at the core of MOO algorithms, being the foundation of
many methods devoted to evaluating the performance and comparing the solutions
to each other~\cite{zitzler2003performance}. However, the frontier is defined by
the objectives, which can be expensive to
compute~\cite{jin2005comprehensive,knowles2008meta,voutchkov2010multi}. This
leads to a variety of surrogate methods that try to approximate the objectives,
e.g. \cite{6466380,6600837}, thus saving computational resources.

Among the surrogates that directly or indirectly estimate the Pareto frontier,
one introduced by Yun et al.~\cite{moo_svm} is the closest to the surrogate
described in this paper. They used a one-class SVM to define a function over the
objective space whose null space describes an approximation of the Pareto
frontier. This function is used to select individuals, since its value increases
as its argument becomes more distant from the frontier, which are then used for
crossover in a genetic algorithm.

Loshchilov et al.~\cite{loshchilov2010mono} presented a similar SVM approach,
but the function learnt is defined over the decision space, which allows direct
comparison with the Pareto frontier approximation without requiring evaluation
of the objectives. This direct comparison can also be achieved with estimates
built over the objective space by integrating surrogates for the objectives.
However, contrary to the one-class SVM that learns a model to fit all samples on
one side of the approximate frontier, the proposed SVM is also able to consider
points that dominate the frontier being approximated, allowing approximation of
multiple Pareto frontiers, each defined by a class of points in non-dominated
sorting~\cite{6766752}.

In a different approach, Loshchilov et al.~\cite{loshchilov2010dominance}
approximated the Pareto dominance instead of the Pareto frontier by using a
rank-based SVM. In this case, instead of providing only the data points, the
algorithm is also informed about the preference for an arbitrary number of
sample pairs and tries to find a function where higher evaluation represents
higher preference. Using the Pareto dominance to establish the preference
between points and learning directly from the decision space, candidate
solutions can be compared in dominance using the learnt function. However, both
\cite{loshchilov2010mono} and \cite{loshchilov2010dominance} try to estimate the
Pareto frontier using generic function approximation models, which do not take
into account the particularities of the Pareto frontier.

It is possible to guarantee that the Pareto frontier's estimate is valid by
building conservative estimates. For instance,
using a binary random field over the objective space to model the boundary
between dominated and non-dominated regions, Da Fonseca and
Fonseca~\cite{da2010attainment} described a theory that can be used to assess
the statistical performance of a stochastic optimization algorithm and compare
different algorithms. The attainment function described in the paper defines the
probability that a run of the stochastic algorithm will dominate the function's
arguments. Although the attainment function is hard to compute, it can be
approximated by multiple runs of the underlying algorithm, which makes it a good
candidate for analyzing the performance statistics of the optimization algorithm
and for performing hypothesis testing between MOO algorithms.

If a single run is considered, then the approximate attainment function
describes a valid estimate of the Pareto frontier and it is defined as the
border of the region dominated by the points provided. Although valid, this
estimate is very conservative and does not interpolate between the points
provided, which means it cannot provide a good idea of the frontier's shape and
any evaluation of new points could be performed using only dominance comparison
with the provided points.

In this paper, we develop a theory that defines necessary and sufficient
conditions for a functional description of a Pareto frontier. Based on this
theory, the search for approximations for the Pareto frontier using surrogate
functions should be constrained to, or at least focused on, the ones that
satisfy the results. If not, the resulting manifold obtained from the function
may have any shape, possibly with many dominated points, which could result in
reduced performance.

Moreover, the theory is developed on the objective space, allowing either
accurate or approximate objective evaluations to be used, without restricting
the format of the objectives' surrogates. If parametric surrogate objectives are
used, their association with the Pareto frontier surrogate can provide feedback
on how to adjust their parameters so that the approximation is closer to the
real objectives.

As an example of how to integrate the theoretical conditions in a surrogate
design, we show how to introduce the theoretical conditions as soft constraints
in Gaussian processes~\cite{rasmussen2006gaussian}, which are nonparametric
models, thus being able to adjust to variable number of samples, and whose
hyper-parameters can be easily optimized.

To validate the hypothesis that surrogate methods that do not consider this
theory may define invalid Pareto frontier approximations, the constrained
Gaussian process is compared to a regular Gaussian process and to an existing
SVM-based surrogate~\cite{moo_svm} and results show that the soft constrained
Gaussian process finds good approximations maximally obeying the constraints
according to the degree of flexibility of the model. On the other hand, the
models that do not take into account the theory can violate greatly and
arbitrarily the conditions for a valid Pareto frontier.

This paper is organized as follows. Section~\ref{sec:moo} introduces the
notation and principles of multi-objective optimization used in this paper.
Section~\ref{sec:conditions} shows the conditions that a function must satisfy
to define a Pareto frontier. These conditions are then used in
Section~\ref{sec:approximations} to build a function to approximate a frontier
given some points on it and the approximation is compared to an existing
surrogate. Finally, Section~\ref{sec:conclusion} summarizes the findings and
points out future directions for research.

\section{Multi-Objective Optimization}
\label{sec:moo}
A multi-objective optimization (MOO) problem is defined by a decision space
$\mathcal X$ and a set of objective functions $g_i(x) \colon \mathcal X \to
\mathcal Y_i, i \in \{1,\ldots,M\}$, where $\mathcal Y_i \subseteq
\R$~\cite{deb2014multi}. Since the framework is the same for maximization or
minimization, we will consider that minimization is desired in all objectives.
For a given point $x$ in the decision space, the point defined by its evaluation
using the objectives $y = (g_1(x), \ldots, g_M(x))$ is its counterpart in the
objective space $\mathcal Y = \mathcal Y_1 \times \cdots \times \mathcal Y_M$.

Although the objective space usually only makes sense when coupled with the
decision space and objectives, which allows for its infeasible region and Pareto
frontier to be defined, we will work only with the objective space in this
paper, which means that the results hold for any problem. We will also consider
that $\mathcal Y = \R^M$, since any restriction for a specific problem is
defined by means of the objectives and decision space constraints, and are
handled transparently.

Furthermore, we assume that the optimal solutions describe a set of $M-1$
manifolds on $\mathbb R^M$, which correspond to curves in the 2D case and
surfaces in the 3D case. Most multi-objective optimization problems have
solutions with this property, with noticeable exceptions, such as: \textit{i})
problems where some of the objectives do not conflict, so that only one of them
should be used in the MOO problem with the other conflicting objectives, while
the optimality of the ignored objectives is guaranteed because they were
redundant; and \textit{ii}) some problems with less decision variables $D$ than
objectives $M$, such as the Viennet function~\cite{coello2002evolutionary}.

Since we are dealing with an optimization problem, we must define operators to
compare solutions, like the operators $<$ and $\le$ are used in the
mono-objective case. In MOO, this operator is the dominance.

\begin{definition}[Dominance]
  \label{def:dominance}
  Let $y$ and $y'$ be points in $\R^M$, the objective space. Then $y$ dominates
  $y'$, denoted $y \preceq y'$, if $y_i \le y'_i$ for all $i$.
\end{definition}

The definition of dominance used in this paper is the same provided
in~\cite{zitzler2003performance}, which allows a point to dominate itself. This
relation is usually called weak dominance, but we call it ``dominance'' for
simplicity, since it is the main dominance relation used in this paper. Another
common definition is to require that $y_i < y'_i$ for at least one $i$, and both
definitions are consistent with the theory developed in this paper.

\begin{definition}[Strong Dominance]
  \label{def:strong_dominance}
  Let $y$ and $y'$ be points in $\R^M$, the objective space. Then $y$ strongly
  dominates $y'$, denoted $y \prec y'$, if $y_i < y'_i$ for all $i$.
\end{definition}

Once defined the comparison operator, we can divide the space $\mathcal Y$ in
three sets: an estimated Pareto frontier, the set of points strongly
dominated by the estimated frontier, and the set of points not strongly
dominated by the estimated frontier.

\begin{definition}[Estimated Pareto Frontier]
  \label{def:frontier}
  A path-connected set of points $F \subset \R^M$ is said to be an estimated
  Pareto frontier if no point in it strongly dominates another point also in
  $F$, that is, $\forall y \in F, \nexists y' \in F \colon y' \prec y$, and
  every point in the objective space except for $F$ either strongly dominates or
  is strongly dominated by a point in $F$, that is, $\forall y \in \R^M - F,
  \exists y' \in F \colon y \prec y' \vee y' \prec y$.
\end{definition}

A set $S$ is path-connected if there is a path joining any two points $x$ and
$y$ in $S$ and a path is defined by a continuous function $p \colon [0,1] \to S$
with $p(0) = x$ and $p(1) = y$. Therefore, if there is a continuous path of
points in $S$ that gets from any $x \in S$ to $y \in S$, then $S$ is
path-connected. Based on this definition, an estimated Pareto frontier $F$
divides the objective space $\R^M$ in three disjoint sets: points strongly
dominated by points in $F$, points that strongly dominates points in $F$, and
$F$ itself.

\begin{definition}[Estimated Strict Pareto Frontier]
  A set of points $F_s \subset \R^M$ is said to be an estimated strict Pareto
  frontier if no point in it dominates another point also in $F_s$, that is,
  $\forall y \in F_s, \nexists y' \in F_s, y' \ne y \colon y' \preceq y$, and
  every point in the objective space except for $F_s$ either dominates or is
  dominated by a point in $F_s$, that is, $\forall y \in \R^M - F_s, \exists y'
  \in F_s \colon y \preceq y' \vee y' \preceq y$.
\end{definition}

\begin{definition}[True Pareto Frontier]
  An estimated strict Pareto frontier $F^*$ is a true Pareto frontier if and
  only if, for all points in $F^*$, there is no other feasible point in the
  objective space that dominates the point in the frontier, that is, $\forall y
  \in F^*, \nexists x \in \mathcal X, g(x) \ne y \colon g(x) \preceq y$.
  Moreover, for a given problem, the true Pareto frontier is unique.
\end{definition}

The estimated Pareto frontier of Definition~\ref{def:frontier} is a
generalization and an approximation of the true Pareto frontier in two ways:
\textit{i}) if the true Pareto frontier is discontinuous, then dominated points
are added so that the estimated Pareto frontier $F$ is path-connected while also
guaranteeing that no point in it strongly dominates any other; and \textit{ii})
the estimated Pareto frontier is simply a set of points that divide the space
into dominated and non-dominated regions, without stating anything about the
optimality of its points.

\begin{figure}[tb]
  \centering
  \psfrag{y1}[c][c]{$y_1$}
  \psfrag{y2}[c][c]{$y_2$}
  \psfrag{D11}[l][l]{\tiny $\overline D$}
  \psfrag{D22}[l][l]{\tiny $D$}
  \psfrag{F11}[l][l]{\tiny $F_s$}
  \psfrag{F22}[l][l]{\tiny $F$}
  \includegraphics[width=0.9\linewidth]{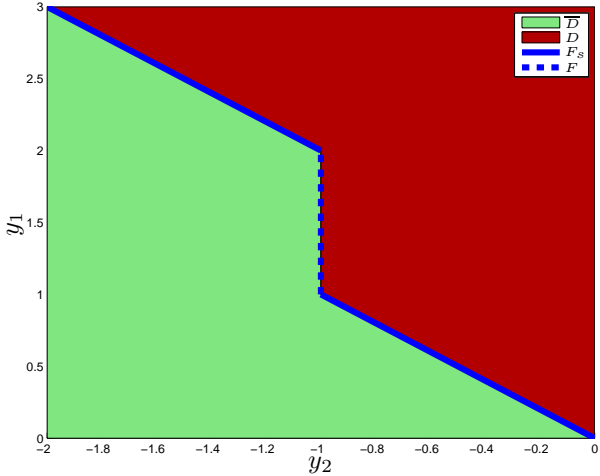}
  \caption{Example of the definitions for a particular multi-objective problem.
    The estimated strict Pareto frontier $F_s$ is shown in a solid blue line,
    the estimated Pareto frontier $F$ includes the solid and dashed blue lines,
    the dominated region $D$ is shown on the top right red area, and the
    non-dominated region $\overline D$ is shown on the bottom left green area.}
  \label{fig:pareto_frontier_example}
\end{figure}

Consider, for instance, a problem where one of the objectives is given by
\begin{equation*}
  g_1(x) =
  \begin{dcases}
    x+1, & x > 1
    \\
    x, & \text{otherwise},
  \end{dcases}
\end{equation*}
and the other is given by $g_2(x) = -x$. Then the true Pareto frontier $F^*$ is
given by
\begin{equation*}
  \begin{aligned}
    F^* =& \{(x+1, -x) \mid x \in \R, x > 1\}
    \\
    &\quad \cup \{(x, -x) \mid x \in \R, x \le 1\},
  \end{aligned}
\end{equation*}
which clearly is not path-connected. However, if we add the set of points $\hat
F = \{(y,-1) \mid y \in (1,2]\}$ to $F^*$, then the resulting path-connected set
$F = F^* \cup \hat F$ satisfies Definition~\ref{def:frontier}, despite the fact
that every point in $\hat F$ is dominated by $(1,1) \in F^*$, but not strongly
dominated by it.

Figure~\ref{fig:pareto_frontier_example} shows an estimated strict Pareto
frontier $F_s$, which coincides with the true Pareto frontier $F^*$ in this
example, and the path-connected estimated Pareto frontier $F$ for this problem.
This makes it clear that the estimated Pareto frontier $F$ can contain the true
Pareto frontier $F^*$, i.e. $F^* \subseteq F$, while providing a path-connected
1D manifold that splits the whole objective space $\R^2$. Of course, these
properties of the estimated Pareto frontier are extensible to $M>2$ objectives.

With the definition of an estimated Pareto frontier, the objective space is
divided into two sets, named dominated and non-dominated sets, also shown in
Fig.~\ref{fig:pareto_frontier_example}.

\begin{definition}[Dominated Set]
  The dominated set $D$ for an estimated Pareto frontier $F$ is the set of all
  points in $\R^M$ where, for each one of them, there is at least one point in
  $F$ that strongly dominates it, that is, $D = \{y \in \R^M \mid \exists y' \in
  F \colon y' \prec y\}$.
\end{definition}

\begin{definition}[Non-Dominated Set]
  The non-dominated set $\overline D$ for an estimated Pareto frontier $F$ is
  the set of all points that are not in $F$ or $D$. This implies that $\overline
  D = \{y \in \R^M \mid \exists y' \in F \colon y \prec y'\}$.
\end{definition}

Note that, from the definition of strong dominance, both $D$ and $\overline D$
are open and unbounded sets, with boundaries defined by the estimated Pareto
frontier $F$. Furthermore, if $F$ contains the true Pareto frontier, then the
points in $\overline D$ are not achievable due to the objectives' definitions.

From the partition of the objective space in three sets, one estimated Pareto
frontier, one dominated and one non-dominated set, we can define a score
function similarly to~\cite{loshchilov2010mono,loshchilov2010dominance}.

\begin{definition}[Score Function]
  \label{def:score}
  A score function $f(y) \colon \R^M \to \R$ for a given estimated Pareto
  frontier $F$ is a function that satisfies
  \begin{align*}
    f(y) = 0, & \quad \forall y \in F,
    \\
    f(y) > 0, & \quad \forall y \in D,
    \\
    f(y) < 0, & \quad \forall y \in \overline D.
  \end{align*}
\end{definition}

Therefore, a score function provides a single value that places its argument in
relation to the estimated frontier. Moreover, for a given estimated Pareto
frontier $F$, there are many possible choices of score functions $f(y)$ that
satisfy the definition and all of them uniquely define $F$ based on their
solution set $f(y) = 0$. This allows a score function to work as a surrogate for
the estimated Pareto frontier.

\section{Necessary and Sufficient Conditions for Surrogate Score Functions}
\label{sec:conditions}
In this section, we will show how a score function $f(y)$ can induce an
estimated Pareto frontier $F$ and the conditions it must satisfy so that the set
it defines is indeed an estimated Pareto frontier, that is, no point in it
strongly dominates any other point in it.

The main theory developed is based on the most general notion of a function $f$,
but the conditions may be hard to evaluate for a general case. Therefore, we
will also provide corollaries that prove the results for functions with
additional constraints, like continuous derivatives. Since some of these results
depend on Taylor approximations and the first derivative at the required points
may be zero, we must define a generalized gradient.

\begin{definition}[Generalized Gradient]
  Let $h \in C^k$, where $C^k$ is the class of functions where the first $k$
  derivatives exist and are continuous, with $k \ge 1$. Let $k^*(h)$ be the
  first non-zero derivative of $h$ evaluated at $0$, that is,
  \begin{equation*}
    k^*(h) = \arg \min_{1 \le i \le k}
    \left(\left.\frac{\text{d}^ih}{\text{d}x^i}\right|_{x=0} \ne 0\right),
  \end{equation*}
  where $k^*(h)$ is not defined if $h$ is a constant function or no $i$
  satisfies the inequality. Then
  \begin{equation*}
    \Delta(h) =
    \begin{dcases}
      0, & \exists a \in \R, \forall x \colon h(x) = a
      \\
      \frac{1}{k^*(h)!}
      \left.\frac{\text{d}^{k^*(h)}h}{\text{d}x^{k^*(h)}}\right|_{x=0}, &
      \text{otherwise}
    \end{dcases}
  \end{equation*}
  is the generalized gradient operator, which is undefined if there is no $i$
  that satisfies the inequality.
\end{definition}

The role of the generalized gradient in the theory to be presented is to avoid
issues with functions that may have null derivative at the points being
evaluated but that are also increasing. Consider, for instance, the function
$f(x) = x^3$, whose gradient is null at $x=0$. This function is strictly
increasing, but the first-order approximation using Taylor series is a constant.
In order to consider small changes in the function's argument, we must use first
non-null derivative, which is the generalized gradient, as it will dominate the
approximation.

The generalized gradient can be used in the Taylor approximation as $h(\delta) =
h(0) + \delta' \Delta(h) + O(\delta'')$, where $0 < \delta \ll 1$, $\delta' =
\delta^{k^*(h)}$, $\delta'' = \delta^{k^*(h)+1}$, and $O(\cdot)$ is the big-O
notation. Since the result is based on $\delta$ being a small value, the exact
power used to compute $\delta'$ is not important for the approximation and the
term $O(\delta'')$ is dominated by the other factors.

The extensions to continuous functions $f$ rely on the generalized gradient of
a single-parameter continuous function $\hat f$, derived from the original $f$,
having different signs for opposite directions. However, it does not hold for
functions where $k^*(\cdot)$ is even.

For example, consider $h(x) = x^2$, which has $k^*(h) = 2$. The Taylor
approximation is given by $h(\delta) \approx \delta^2 \Delta(h(x)) = 2 \delta^2
= \delta^2 \Delta(h(-x)) \approx h(-\delta)$, which does not give different
signs to different directions of $x$. Therefore, the two constraints on
$\Delta(\hat f)$ defined in the corollaries that follow can be viewed as a
single constraint on $\Delta(\hat f)$ plus the constraint that $k^*(\hat f)$ is
odd.
\subsection{Necessary Conditions}
The necessary conditions derived are direct applications of the estimated Pareto
frontier's definition and establish the basic ground on how to define a function
$f$ from a given estimated frontier.

\begin{restatable}[General Necessity]{lem}{generalNecessity}
  \label{lem:generalNecessity}
  Let $F$ be an estimated Pareto frontier. Let $f(y) \colon \R^M \to \R$ be a
  score function for $F$.
  Then $f(y + \delta u) > 0$ and $f(y - \delta u) < 0$ for all $y \in F$, $u \in
  (0,1]^M$, and $\delta \in \R, \delta > 0$.
\end{restatable}
\begin{proof}
  Assume there are $y$, $u$, and $\delta > 0$ such that $f(y + \delta u) \le 0$.
  Let $y' = y + \delta u$, so that $y \prec y'$.

  If $f(y') < 0$, then from the definition of a score function there is some
  $y^* \in F$ such that $y' \prec y^*$. From the transitivity of dominance, we
  have that $y \prec y' \prec y^*$, which is a contradiction, since the point
  $y^*$ in the frontier cannot strongly dominate the point $y$ also in the
  frontier. Then we must have $f(y') = 0$, which means $y' \in F$ and also
  creates a contradiction.

  Assume that $f(y - \delta u) \ge 0$, and let $y'' = y - \delta u$. Then we can
  similarly prove that it also creates a contradiction.

  Therefore, there are no such $y$, $u$, and $\delta$ with $f(y + \delta u) \le
  0$ or $f(y - \delta u) \ge 0$.
\end{proof}

This result is intuitive, since moving $\delta$ in direction $u$ from $y$ we
enter either $D$ or $\overline D$. If the function has the required derivatives,
then the following result holds.

\begin{restatable}[Differentiable Necessity]{cor}{diffNecessity}
  \label{cor:differentiableNecessity}
  Let $F$ be an estimated Pareto frontier. Let $f(y) \colon \R^M \to \R$ be a
  score function for $F$.
  Let $\hat f_{y,u}^+(x) = f(y + xu)$ and $\hat f_{y,u}^-(x) = f(y - xu)$, with
  $x \in [0, \infty)$.
  Let $\Delta(\hat f_{y,u}^+)$ and $\Delta(\hat f_{y,u}^-)$ be defined for all
  $y \in F$ and $u \in (0,1]^M$.
  Then $\Delta(\hat f_{y,u}^+) > 0$ and $\Delta(\hat f_{y,u}^-) < 0$ for all $y
  \in F$ and $u \in (0,1]^M$.
\end{restatable}
\begin{proof}
  Since $f$ satisfies all conditions from Lemma~\ref{lem:generalNecessity}, we
  have that $f(y + \delta u) > 0$ and $f(y - \delta u) < 0$ for all $y$, $u$,
  and $\delta > 0$.

  In particular, let $\delta \ll 1$. Approximating using Taylor series, we have
  that $f(y + \delta u) \approx f(y) + \delta' \Delta(\hat f_{y,u}^+) > 0$ and
  $f(y - \delta u) \approx f(y) + \delta' \Delta(\hat f_{y,u}^-) < 0$, where
  $\delta'$ is the appropriate power of $\delta$ for the expansion. Since $f(y)
  = 0$ and $\delta' > 0$, then $\Delta(\hat f_{y,u}^+) > 0$ and $\Delta(\hat
  f_{y,u}^-) < 0$ must hold.
\end{proof}

Although this corollary may appear to provide weaker guarantees on $f$, its
proof shows that the inequality constraints on the generalized gradient is
equivalent to the direct inequalities on the function defined in the previous
lemma.
\subsection{Sufficient Conditions}
Once defined how the estimated Pareto frontier relates to a given score
function, we will show that a function that satisfies the results of the
previous lemma and corollary in fact uniquely defines an estimated Pareto
frontier $F$.

\begin{restatable}[General Sufficiency]{lem}{generalSufficiency}
  \label{lem:generalSufficiency}
  Let $f(y) \colon \R^M \to \R$ be a function. Let $F = \{y \in \R^M \mid f(y) =
  0\}$ be a path-connected set. Let $f(y + \delta u) > 0$ and $f(y - \delta u) <
  0$ for all $y \in F$, $u \in (0,1]^M$, and $\delta \in \R, \delta > 0$. Then
  $F$ is an estimated Pareto frontier.
\end{restatable}
\begin{proof}
  For $F$ to be an estimated Pareto frontier, we have to prove that for any
  $y,y' \in F, y \ne y'$ we have $y \nprec y'$. Assume there are $y$ and $y'$ in
  $F$ such that $y \prec y'$.

  Let $u = y' - y$ and $\delta  = 1$. Then we have $f(y+\delta u) = f(y') = 0$,
  which violates the first inequality on $f(\cdot)$. Alternatively, we have
  $f(y' - \delta u) = f(y) = 0$, which violates the second inequality.

  Therefore, there are no $y$ and $y'$ in $F$ such that $y \prec y'$, and $F$ is
  an estimated Pareto frontier.
\end{proof}

The restrictions on $f(y \pm \delta u)$ may be hard to verify in general, since
they must be valid for all $\delta$. However, if the function has the
appropriate derivatives, then it becomes easier to check if it satisfies the
requirements.

\begin{restatable}[Differentiable Sufficiency]{cor}{diffSufficiency}
  \label{cor:differentiableSufficiency}
  Let $f(y) \colon \R^M \to \R$ be a function. Let $F = \{y \in \R^M \mid f(y) =
  0\}$ be a path-connected set.
  Let $\hat f_{y,u}^+(x) = f(y + xu)$ and $\hat f_{y,u}^-(x) = f(y - xu)$, with
  $x \in [0, \infty)$.
  Let $\Delta(\hat f_{y,u}^+) > 0$ and $\Delta(\hat f_{y,u}^-) < 0$ for all $y
  \in F$ and $u \in (0,1]^M$.
  Then $F$ is an estimated Pareto frontier.
\end{restatable}
\begin{proof}
  To use Lemma~\ref{lem:generalSufficiency}, we must prove that $f(y + \delta u)
  > 0$ and $f(y - \delta u) < 0$ for all $y \in F$, $u \in (0,1]^M$, and $\delta
  \in \R, \delta > 0$.

  Suppose there is some $y$, $u$, and $\delta$ in the domain such that $f(y +
  \delta u) = 0$. Moreover, let $\delta$ be the smallest value for which this
  happens for a given $y$ and $u$.
  Let $0 < \epsilon \ll 1$ and $\epsilon < \delta$. Then $f(y + \epsilon u)
  \approx f(y) + \epsilon' \Delta(\hat f_{y,u}^+) > 0$ and $f((y + \delta u) -
  \epsilon u) \approx f(y + \delta u) + \epsilon' \Delta(\hat f_{y,u}^-) < 0$,
  where $\epsilon'$ is the appropriate power of $\epsilon$ for the
  approximation.
  However, $f(\cdot)$ cannot go from positive to negative without passing
  through $0$ due to its continuity. Then there must be some $\delta' < \delta$
  such that $f(y + \delta' u) = 0$, which contradicts the definition of
  $\delta$.

  Therefore, the first inequality on Lemma~\ref{lem:generalSufficiency} holds.
  We can use a similar method to prove the second inequality, and then use the
  lemma.
\end{proof}

Again, this corollary shows the equivalence between the inequalities on the
function and on the generalized gradient.
\subsection{Necessary and Sufficient Conditions}
Since the symmetry between Lemmas~\ref{lem:generalNecessity}
and~\ref{lem:generalSufficiency} is clear, we can build a theorem to merge those
two and provide necessary and sufficient conditions for defining an estimated
Pareto frontier $F$ from a score function $f(y)$.

\begin{restatable}[General Score Function]{thm}{generalPareto}
  \label{thm:generalIff}
  Let $f(y) \colon \R^M \to \R$ be a function. Let $F = \{y \in \R^M \mid f(y) =
  0\}$ be a path-connected set. Let $D = \{y \in \R^M \mid \exists y' \in F
  \colon y' \prec y\}$ and $\overline D = \R^M \backslash (F \cup D)$. Let $f(y)
  > 0, \forall y \in D$, and $f(y) < 0, \forall y \in \overline D$. Then $F$ is
  an estimated Pareto frontier if and only if $f(y + \delta u) > 0$ and $f(y -
  \delta u) < 0$ for all $y \in F$, $u \in (0,1]^M$, and $\delta \in \R, \delta
  > 0$.
\end{restatable}
\begin{proof}
  Assume that the constraints on $f$ are valid. Then, from
  Lemma~\ref{lem:generalSufficiency}, we have that $F$ is an estimated Pareto
  frontier. Now assume that $F$ is an estimated Pareto frontier. Then, from
  Lemma~\ref{lem:generalNecessity}, we have that the constraints on $f$ are
  valid.
\end{proof}

Instead of requiring knowledge of the sign of $f(y)$ over the sets, we can use a
more strict definition, requiring continuity, to guarantee that the result
holds.

\begin{restatable}[Continuous Score Function]{cor}{generalContPareto}
  \label{cor:generalContIff}
  Let $f(y) \colon \R^M \to \R$ be a continuous function where there are points
  $v_+$ and $v_-$ such that $f(v_+) > 0$, $f(v_-) < 0$, and $v_- \prec v_+$. Let
  $F = \{y \in \R^M \mid f(y) = 0\}$ be a path-connected set. Then $F$ is an
  estimated Pareto frontier if and only if $f(y + \delta u) > 0$ and $f(y -
  \delta u) < 0$ for all $y \in F$, $u \in (0,1]^M$, and $\delta \in \R, \delta
  > 0$.
\end{restatable}
\begin{proof}
  Assume that $F$ is an estimated Pareto frontier. Assume that there are $y,y'
  \in D = \{y \in \R^M \mid \exists y' \in F \colon y' \prec y\}$ such that
  $f(y) > 0$ and $f(y') < 0$. From the continuity of $f$, we have that there is
  some $z \in D$ such that $f(z) = 0$. However, since $f(z) = 0$, it is in $F$.
  From the definition of $D$, there is some $z' \in F$ such that $z' \prec z$,
  which violates the assumption that $F$ is an estimated Pareto frontier.
  Therefore, all points in $D$ have the same sign over $f$. The same can be
  shown for $\overline D$.

  Since $v_- \prec v_+$, we have that $v_+ \in D$ and $v_- \in \overline D$.
  Then $f$ satisfies all conditions from Theorem~\ref{thm:generalIff}.
\end{proof}

Again, we can replace the constraints on $f(y \pm \delta u)$ by the constraint
on the generalized gradient.

\begin{restatable}[Differentiable Score Function]{cor}{diffPareto}
  \label{cor:differentiableIff}
  Let $f(y) \colon \R^M \to \R$ be a function where there are points $v_+$ and
  $v_-$ such that $f(v_+) > 0$, $f(v_-) < 0$, and $v_- \prec v_+$. Let $F = \{y
  \in \R^M \mid f(y) = 0\}$ be a path-connected set.
  Let $\hat f_{y,u}^+(x) = f(y + xu)$ and $\hat f_{y,u}^-(x) = f(y - xu)$.
  Let $\Delta(\hat f_{y,u}^+)$ and $\Delta(\hat f_{y,u}^-)$ be defined for all
  $y \in F$ and $u \in (0,1]^M$.
  Then $F$ is an estimated Pareto frontier if and only if $\Delta(\hat
  f_{y,u}^+) > 0$ and $\Delta(\hat f_{y,u}^-) < 0$ for all $y \in F$ and $u \in
  (0,1]^M$.
\end{restatable}
\begin{proof}
  We can use Corollary~\ref{cor:generalContIff} to show that the restrictions on
  $f(y \pm \delta u)$ must hold. From
  Corollaries~\ref{cor:differentiableNecessity}
  and~\ref{cor:differentiableSufficiency}, we know that the restrictions on
  $\Delta(\hat f_{y,u}^\pm)$ are the same as the restrictions on $f(y \pm \delta
  u)$, so this corollary is valid.
\end{proof}

\section{Learning Surrogate Functions from Samples}
\label{sec:approximations}
After showing what conditions the function $f$ must satisfy, one could ask how
to build such function for a given problem and specially how to learn one from a
given set of non-dominated points. This can be a hard question to answer in
general, but we can provide an additional lemma that can help in many cases.

\begin{restatable}[Strictly Increasing Sufficiency]{lem}{monotonicSufficiency}
  \label{lem:monotonic_sufficiency}
  Let $f(y) \colon \R^M \to \R$ be a strictly increasing function on each
  coordinate. Let $F = \{y \in \R^M \mid f(y) = 0\}$. Then $F$ is an estimated
  Pareto frontier.
\end{restatable}
\begin{proof}
  For $F$ to be an estimated Pareto frontier, we have to prove that for any
  $y,y' \in F, y \ne y'$ we have $y \nprec y'$. Assume there are $y$ and $y'$ in
  $F$ such that $y \prec y'$.

  Let $P = (p_0 = y, p_1,\ldots,p_{M-1},p_M = y')$ be a path between $y$ and
  $y'$ that increments only one coordinate at a time. Since $f$ is strictly
  increasing, we have that $f(p_i) < f(p_{i+1})$. Thus $f(y) < f(y')$, which
  contradicts the premise that $f(y) = f(y') = 0$ because they are both in the
  frontier.

  Therefore, there are no $y$ and $y'$ in $F$ where $y \prec y'$ and $F$ is an
  estimated Pareto frontier.
\end{proof}

Note that, because $f$ is strictly increasing, there is no point in $F$ that
even dominates another point in $F$, which was allowed in
Definition~\ref{def:frontier}. This restriction can be relaxed to be only
monotonically non-decreasing if one can guarantee that $f(y) = 0$ is only a
manifold, and not a subspace with volume. If $f(y) = 0$ is a subspace, then we
can find two points in it where one dominates the other, which violates the
basic definition of an estimated Pareto frontier. For instance, a function that
is monotonically non-decreasing and is constant in at most one dimension at a
time does not create a subspace on $f(y) = 0$.

Nonetheless, this lemma can be used as a guide on how to build a function for
the general case. We will build a model that tries to approximate an estimated
Pareto frontier from a few of its samples using an approximated monotonically
increasing function based on Gaussian processes.
\begin{figure*}[t]
  \centering
  \begin{subfigure}[b]{0.45\linewidth}
    \psfrag{x}[c][c]{$x$}
    \psfrag{y}[c][c]{$f(x)$}
    \includegraphics[width=\linewidth]{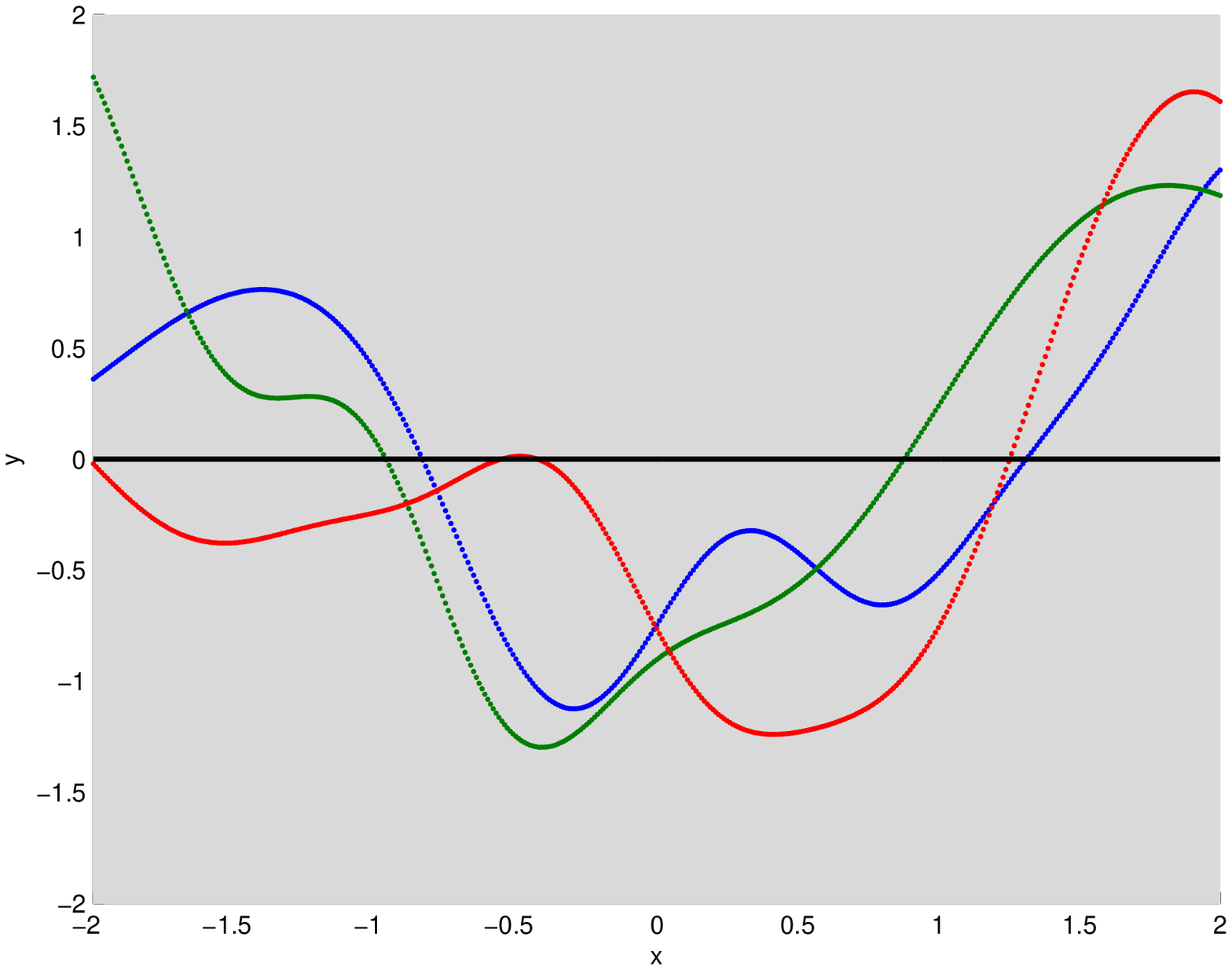}
    \caption{Before observations}
    \label{fig:gp_example_prior}
  \end{subfigure}
  ~
  \begin{subfigure}[b]{0.45\linewidth}
    \psfrag{x}[c][c]{$x$}
    \psfrag{y}[c][c]{$f(x)$}
    \includegraphics[width=\linewidth]{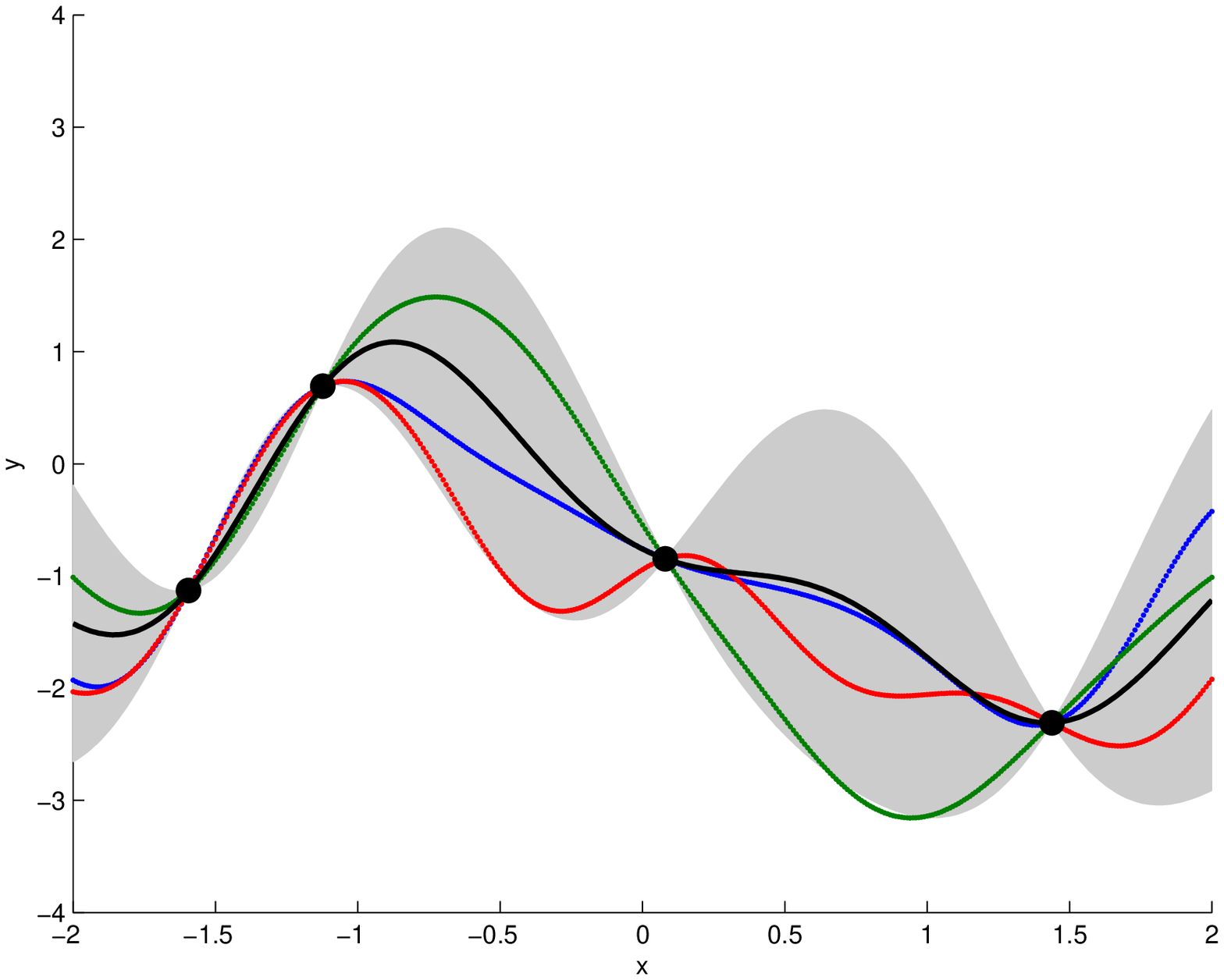}
    \caption{After observations}
    \label{fig:gp_example_posterior}
  \end{subfigure}
  \caption{Function distribution using a Gaussian process. Before the
    observations, the distribution is the same over all the space. After the
    observations, the distribution adapts to constraint the possible functions.
    The distribution mean is given by the black line and the 95\% confidence
    interval is given by the shadowed region. Three function samples are also
    provided for each case.}
  \label{fig:gp_example}
\end{figure*}
\subsection{Gaussian Process As a Function Approximation Problem}
\label{sec:approximations:gp_review}
Since the model should have enough flexibility to fit the given samples, an
appropriate choice for a surrogate function is a Gaussian process, which always
has enough capacity to fit the data. Before describing how a Gaussian process is
used to approximate the Pareto frontier, we provide the reader with an overview
of how they work. For a more detailed description, we refer the reader
to~\cite{rasmussen2006gaussian}.

A Gaussian process (GP) is a generalization of the multivariate normal
distribution to infinite dimensions and can be used to solve a regression
problem. A GP defines a probability distribution over functions, such that the
outputs are jointly normally distributed.

To better understand this concept, consider an infinite column vector $y \in
\R^\infty$ and an infinite matrix $x \in \R^{\infty\times D}$. Then a function
$f \colon \R^D \to \R$ can be described by associating the row indexes, such
that $f(x_i) = y_i$. The GP relies on the fact that the relationship between $x$
and $y$ can be written as:
\begin{equation}
  \label{eq:gaussian_process_basic}
  y \sim \N(\mu(x), K(x)),
\end{equation}
which states that all dimensions of $y$ are distributed according to a
multivariate normal distribution with mean $\mu(x)$ and covariance $K(x)$.
Moreover, the mean for a given dimension is given by $\E[y_i] = \mu(x_i)$ and
the covariance is given by $\Cov(y_i,y_j) = k(x_i, x_j)$, where $k(\cdot,\cdot)$
is a positive semi-definite kernel function.

Although continuous functions, and thus Gaussian processes, are defined for an
infinite number of points, which caused the vectors $x$ and $y$ to have infinite
dimensions, only a finite number of observations are actually made in practice.
Let $N$ be such number of observations. Then, by the marginalization property of
the multivariate normal distribution, we only have to consider $N$ observed
dimensions of $x$ and $y$. Furthermore, the finite-dimension version of $y$ is
still normally distributed according to Eq.~\eqref{eq:gaussian_process_basic}
when considering only the observed dimensions.

Usual choices for the mean and covariance functions are the null
mean~\cite{rasmussen2006gaussian}, such that $\mu(x) = 0$, and the squared
exponential kernel, defined by:
\begin{equation}
  \label{eq:rbf_kernel}
  k(x, x') = \eta^2 \exp \left(-\frac{1}{2} \sum_{i=1}^D
  \frac{(x_i - x_i')^2}{\rho_i^2}\right),
\end{equation}
where $\eta,\rho_i > 0$ and $\rho_i$ are the scale parameters, which define a
representative scale for the smoothness of the function.

The choice of the kernel function establishes the shape and smoothness of the
functions defined by the GP, with the squared exponential kernel defining
infinitely differentiable functions. Other choices of kernel are possible and
provide different compromises regarding the shape of the function being
approximated, such as faster changes and periodicity of values. However, in
order to use the monotonicity constraints introduced in
Section~\ref{sec:gp_monotonic}, the kernel has to be at least twice
differentiable, which limits the kernels that can be used.

Figure~\ref{fig:gp_example_prior} shows the prior distribution over functions
using the squared exponential kernel with $\eta = 1$, $\rho = 0.5$, $D=1$, and
the zero mean. This highlights the fact that the GP defines a distribution over
functions, not a unique function. Three sample functions from this GP are also
shown in the same figure. Note that the functions are not shown as continuous,
which would require an infinite number of points, but as finite approximations.

To use the GP to make predictions, the observed values of $x$ are split into a
training set $X$, whose output $Y$ is known, and a test set $X_*$, whose output
$Y_*$ we want to predict. Since all observations are jointly normally
distributed, we have that the posterior distribution is given by:
\begin{subequations}
\label{eq:gp_posterior_noiseless}
\begin{gather}
  Y_*|X_*,X,Y \sim \N(\mu_*, \Sigma_*)
  \\
  \mu_* = K(X_*,X) K(X,X)^{-1} Y
  \\
  \Sigma_* = K(X_*,X_*) - K(X_*,X) K(X,X)^{-1} K(X,X_*),
\end{gather}
\end{subequations}
where $K(\cdot,\cdot)$ are matrices built by computing the kernel function for
each combination of the arguments values.

The posterior distribution for the previous GP, after four observations marked
as black dots, is shown in Fig.~\ref{fig:gp_example_posterior}. Note that the
uncertainty around the observed points is reduced due to the observation
themselves, and the mean function passes over the points, as expected.
Again, three functions are sampled from the posterior, and all agree on the
value the function must assume over the observations.

In order to avoid some numerical issues and to consider noisy observations, we
can assume that the covariance has a noisy term. Assuming that $y_i = f(x_i) +
\epsilon_i$, where $\epsilon_i$ is normally distributed with zero mean and
variance $\sigma^2$, then the covariance of the observations is given by
$\Cov(y_i,y_j) = k(x_i,x_j) + \sigma^2 \delta_{ij}$. The noiseless value $l_i =
f(x_i)$ can then be estimated by:
\begin{subequations}
\label{eq:gp_posterior_noised}
\begin{gather}
  L_*|X_*,X,Y \sim \N(\mu_*, \Sigma_*)
  \\
  \mu_* = K(X_*,X) \Omega Y
  \\
  \Sigma_* = K(X_*,X_*) - K(X_*,X) \Omega K(X,X_*)
  \\
  \Omega = \left[K(X,X) + \sigma^2 I\right]^{-1},
\end{gather}
\end{subequations}
which is similar to Eq.~\eqref{eq:gp_posterior_noiseless}, except for the added
term in $\Omega$ corresponding to the noise.
\begin{figure*}[p]
  \centering
  \begin{subfigure}[b]{0.45\linewidth}
    \psfrag{x}[c][c]{$y_1$}
    \psfrag{y}[c][c]{$y_2$}
    \includegraphics[width=\linewidth]{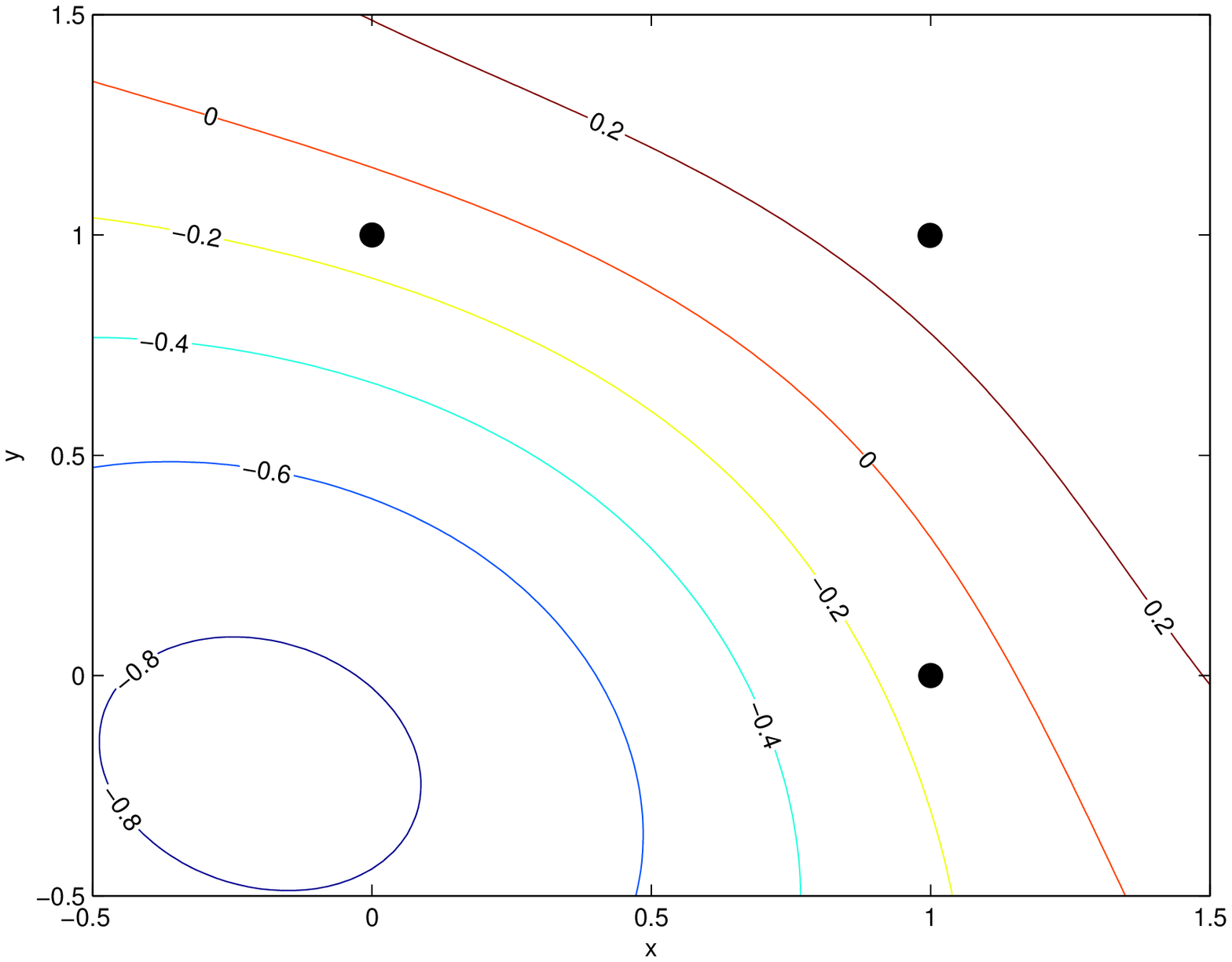}
    \caption{Concave, $\beta \to \infty$}
  \end{subfigure}
  ~
  \begin{subfigure}[b]{0.45\linewidth}
    \psfrag{x}[c][c]{$y_1$}
    \psfrag{y}[c][c]{$y_2$}
    \includegraphics[width=\linewidth]{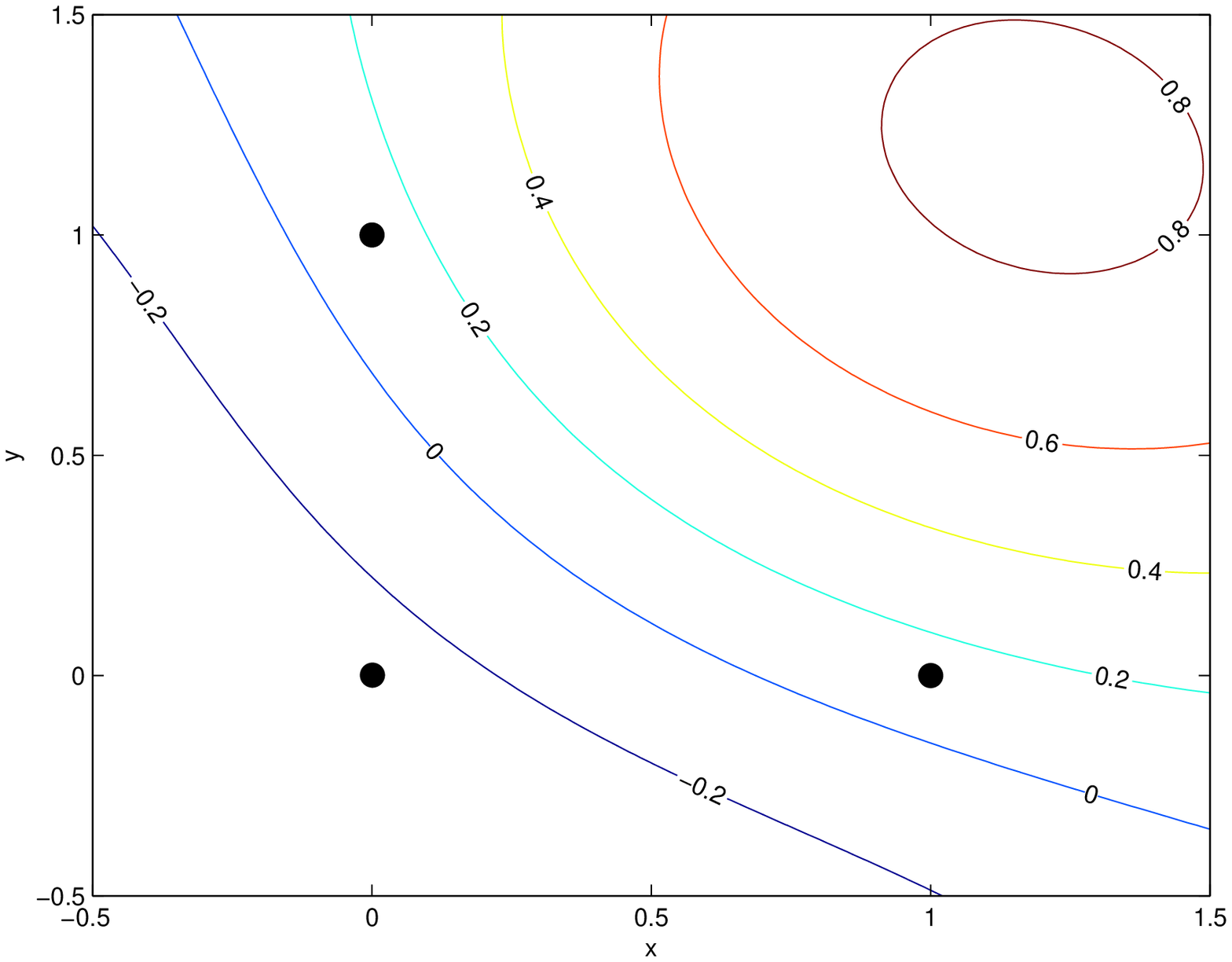}
    \caption{Convex, $\beta \to \infty$}
  \end{subfigure}
  \\
  \begin{subfigure}[b]{0.45\linewidth}
    \psfrag{x}[c][c]{$y_1$}
    \psfrag{y}[c][c]{$y_2$}
    \includegraphics[width=\linewidth]{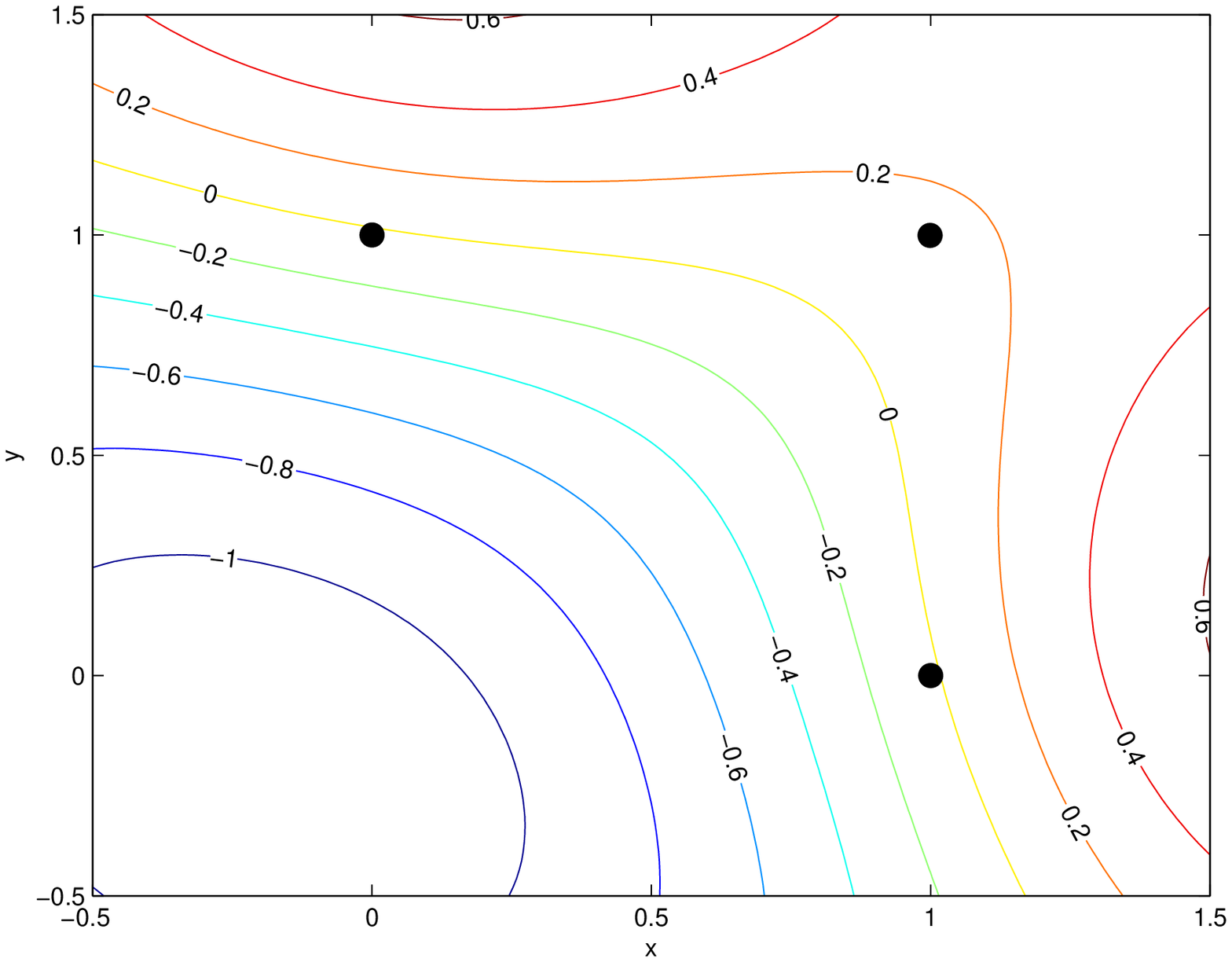}
    \caption{Concave, $\beta = 0.1$}
  \end{subfigure}
  ~
  \begin{subfigure}[b]{0.45\linewidth}
    \psfrag{x}[c][c]{$y_1$}
    \psfrag{y}[c][c]{$y_2$}
    \includegraphics[width=\linewidth]{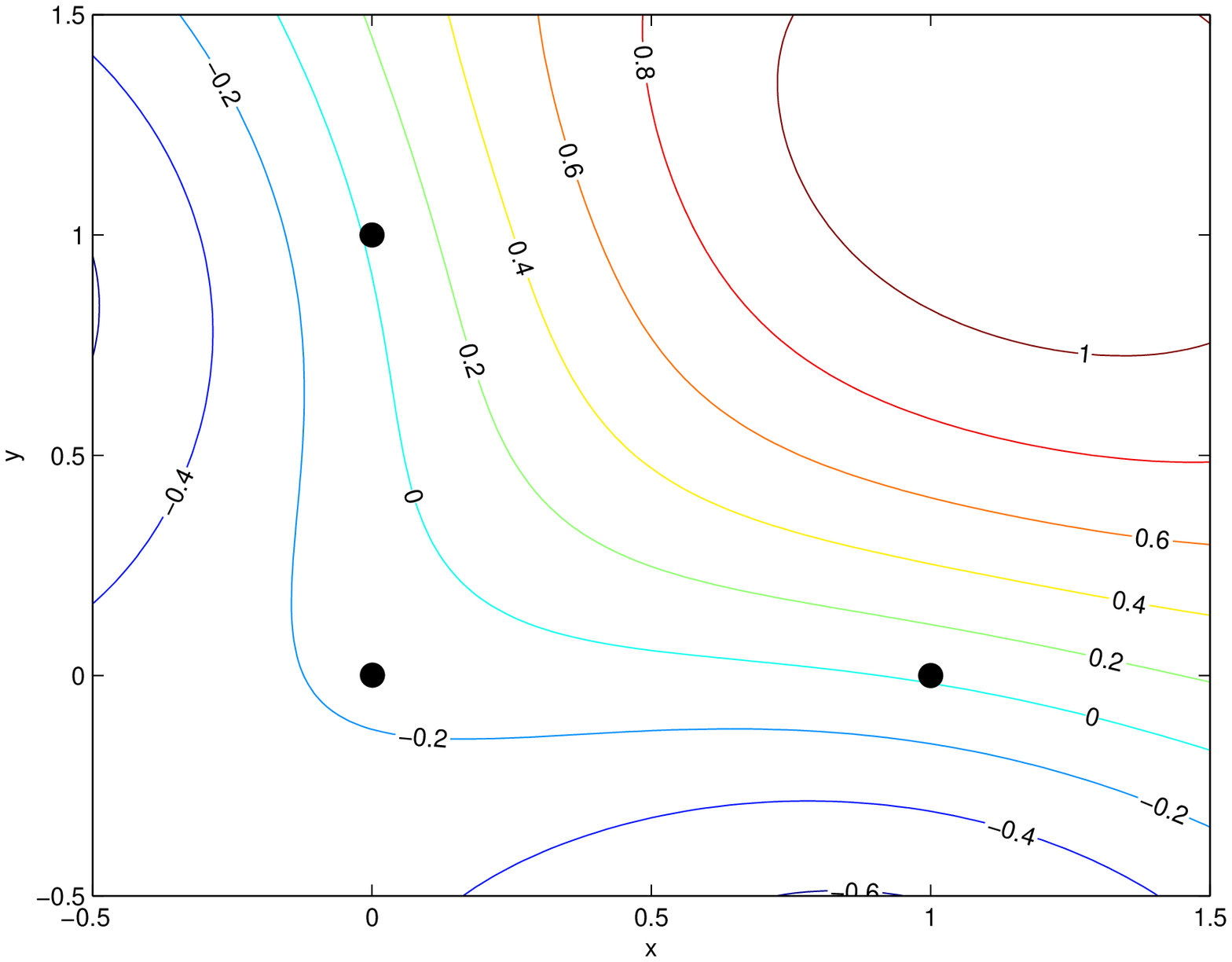}
    \caption{Convex, $\beta = 0.1$}
  \end{subfigure}
  \\
  \begin{subfigure}[b]{0.45\linewidth}
    \psfrag{x}[c][c]{$y_1$}
    \psfrag{y}[c][c]{$y_2$}
    \includegraphics[width=\linewidth]{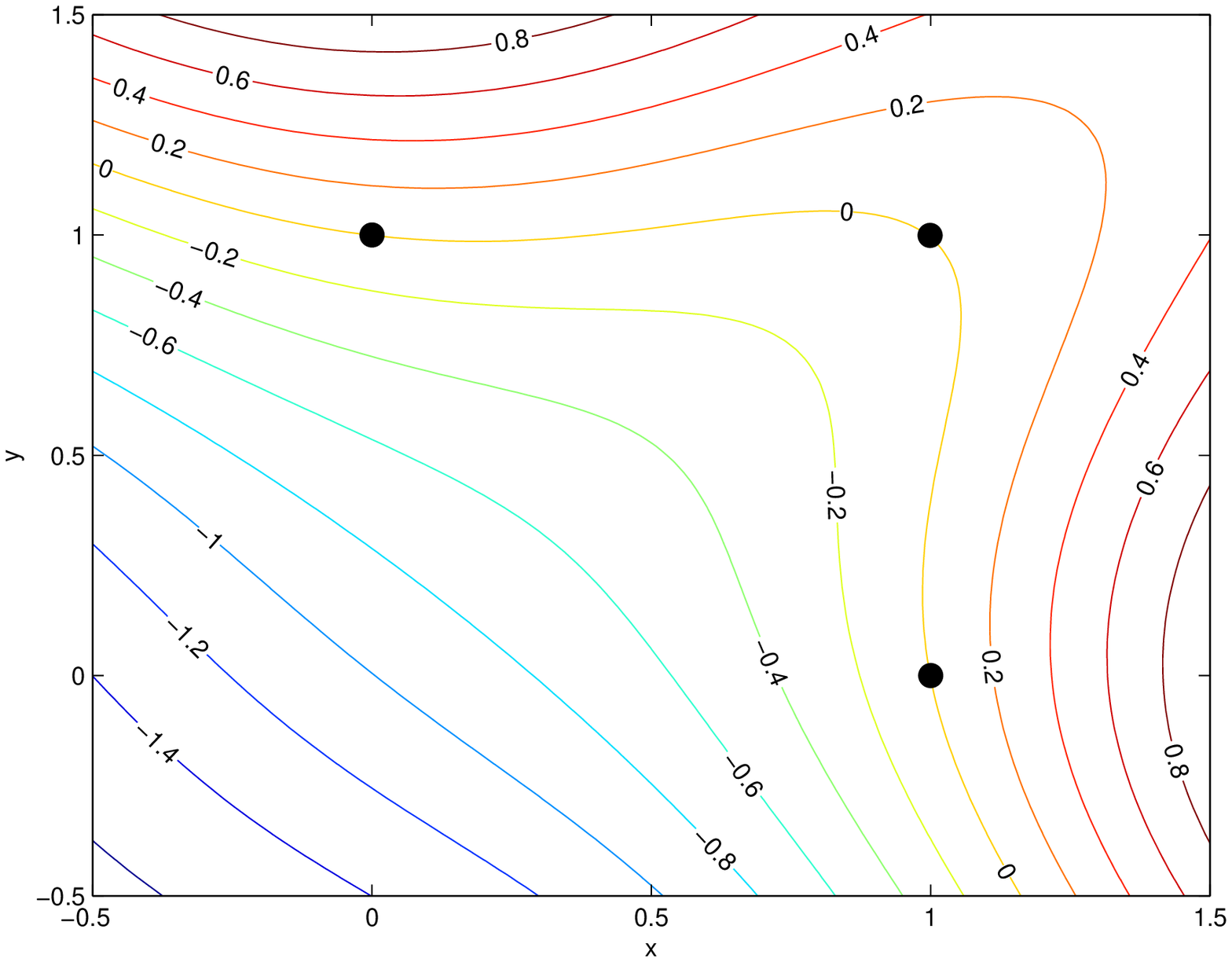}
    \caption{Concave, $\beta \to 0$}
  \end{subfigure}
  ~
  \begin{subfigure}[b]{0.45\linewidth}
    \psfrag{x}[c][c]{$y_1$}
    \psfrag{y}[c][c]{$y_2$}
    \includegraphics[width=\linewidth]{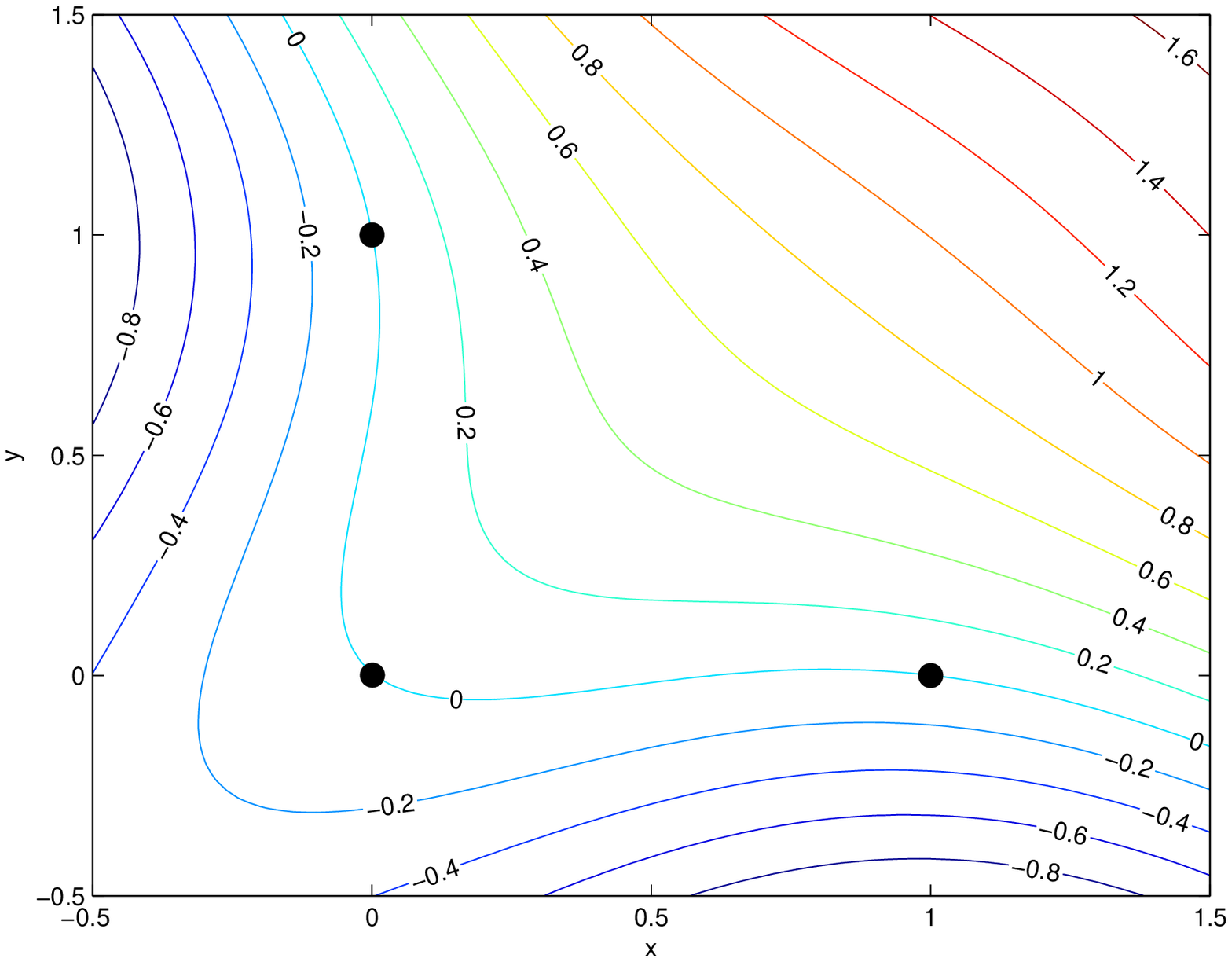}
    \caption{Convex, $\beta \to 0$}
  \end{subfigure}
  \caption{Contours for the $f(y)$ learned using a Gaussian process with
  derivative constraint. The black dots are the frontier points provided.}
  \label{fig:gp_derivative}
\end{figure*}
\begin{figure*}[t]
  \centering
  \begin{subfigure}[b]{0.45\linewidth}
    \psfrag{x}[t][t]{\small $y_1$}
    \psfrag{y}[c][c]{\small $y_2$}
    \includegraphics[width=\linewidth]{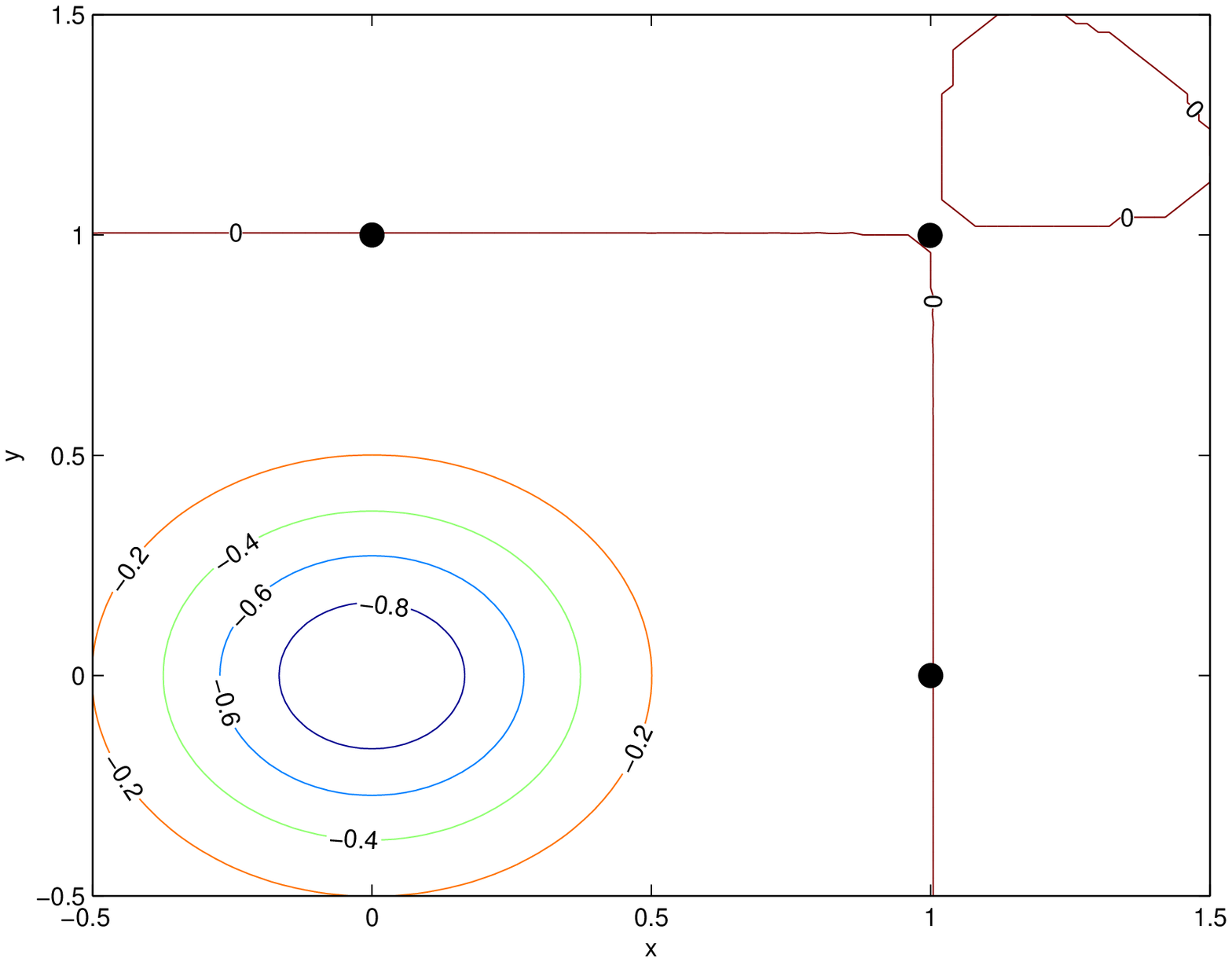}
    \caption{Concave, $\beta = 0.1$}
  \end{subfigure}
  ~
  \begin{subfigure}[b]{0.45\linewidth}
    \psfrag{x}[t][t]{\small $y_1$}
    \psfrag{y}[c][c]{\small $y_2$}
    \includegraphics[width=\linewidth]{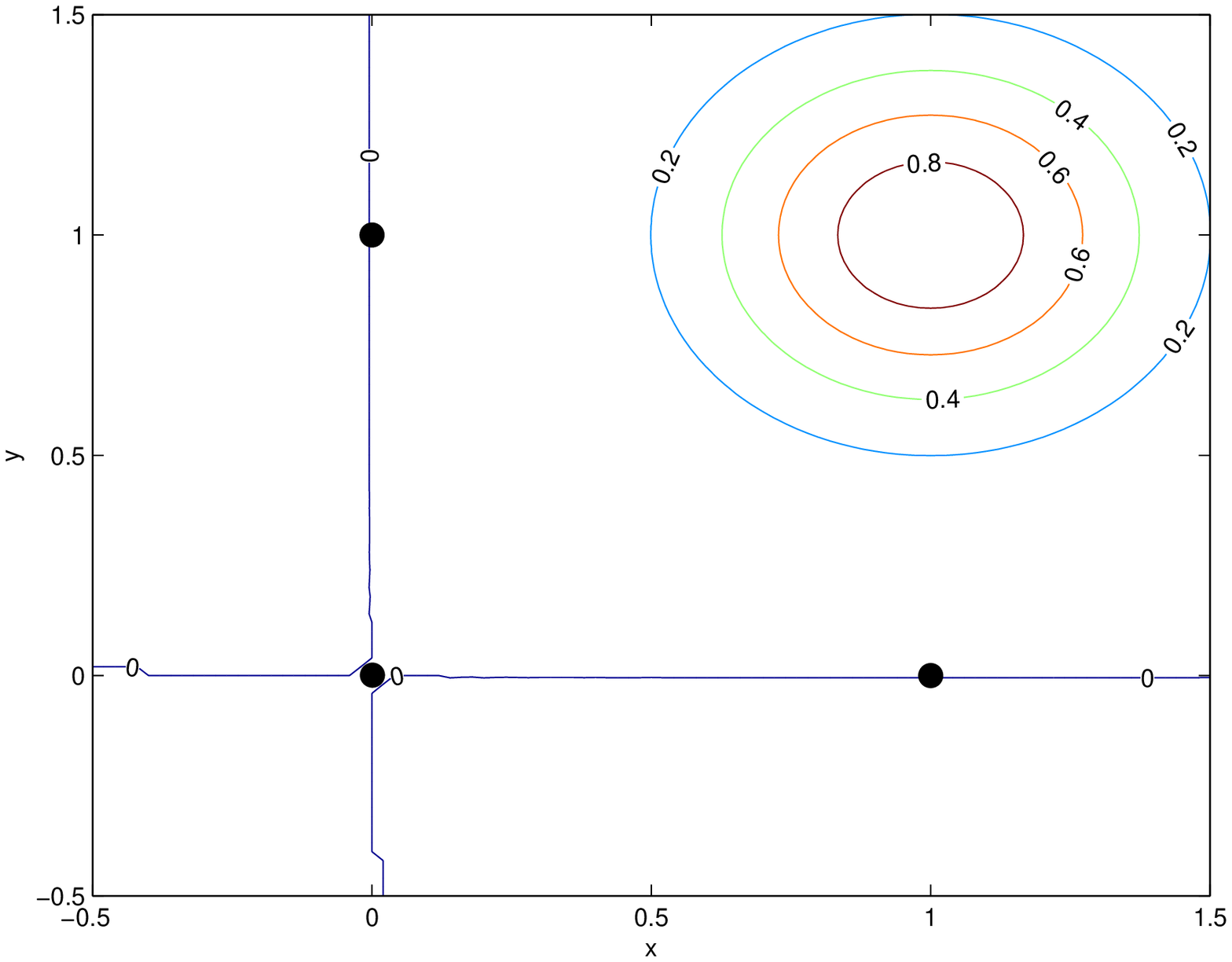}
    \caption{Convex, $\beta = 0.1$}
  \end{subfigure}
  \caption{Contours for the $f(y)$ learned using standard Gaussian process. The
  black dots are the frontier points provided.}
  \label{fig:gp}
\end{figure*}
\subsection{Gaussian Processes with Monotonicity Soft Constraint as Surrogates}
\label{sec:gp_monotonic}
Just like in the previous section, we consider the null mean function $\mu(x) =
0$ and the squared exponential kernel defined in Eq.~\eqref{eq:rbf_kernel}.
Since we are mapping from the objective space $\R^M$ to a value in $\R$,
according to Definition~\ref{def:score}, the input values are the objectives $y$
and the outputs the scores $z$.

Let $Y \in \R^{N \times M}$ be a set of $N$ input points and $Z \in \R^N$ their
desired targets for training. We define the latent variable $L$ between the two,
such that
\begin{equation*}
  L|X \sim \N(0, K(Y, Y)),
\end{equation*}
where $K(Y, Y)_{i,j} = k(y_i, y_j)$. The latent variable then produces the
observed values $Z$ through
\begin{equation*}
  Z|L \sim \N(L, \sigma^2 I),
\end{equation*}
where $I$ is the identity matrix.

This model is the same as the one described in
Section~\ref{sec:approximations:gp_review}. However, only the mean prediction
will be used in this paper to describe the estimated Pareto frontier. Moreover,
we will show how changing the allowed noise level $\sigma$  affects the Pareto
frontier approximation.

Besides the observations of $f(y)$ at the desired points, the GP framework also
accepts observations of its derivative, since differentiation is a linear
operator~\cite{o1992some,rasmussen2003gaussian}, that is, the derivative of a GP
is also a Gaussian process. However, since we do not know the desired value of
the gradient, only that it should be positive, from
Corollary~\ref{cor:differentiableIff} and Lemma~\ref{lem:monotonic_sufficiency},
forcing an arbitrary value may lead to reduced performance.

Another option is to introduce a probability distribution over the gradient in
order to favor positive values, introducing monotonicity
information~\cite{riihimaki2010gaussian}. This new distribution can be viewed
as adding constraints to the Gaussian process, making it feasible to include the
monotonicity information to the existing framework.

Ideally, the probability distribution over the gradient is the step function,
which provides a probability of zero if the gradient is negative and the same
probability for all positive gradients. However, the step function defines a
hard threshold and does not allow small errors, which can cause some problems
for the optimization. Therefore, a smooth function that approximates the step is
used to define a soft constraint over the gradient.

Let $m^{(i)}_{d_i}$ be the indication that the $i$-th sample is monotonic in the
direction $d_i$. Then the following probability distribution can be used to
approximate the step function:
\begin{subequations}
\label{eq:monotonicity}
\begin{gather}
  p\left(m^{(i)}_{d_i} \left| \frac{\partial l^{(i)}}{\partial y_{d_i}} \right.
  \right)
  =
  \Phi\left(\frac{\partial l^{(i)}}{\partial y_{d_i}} \frac{1}{\nu}\right)
  \\
  \Phi(v) = \int_{-\infty}^v \N(t|0,1) \text{d}t,
\end{gather}
\end{subequations}
where we assume the probit function $\Phi(\cdot)$ as the derivative probability.
Since the probit is a cumulative distribution function, its value ranges from
$0$ to $1$ and it is monotonically increasing, which makes it a good
approximation for the step function. The parameter $\nu$ allows us to define how
strict the distribution should be, with $\nu \to 0$ approximating the step
function or a hard constraint. In this paper, following the suggestion
of~\cite{riihimaki2010gaussian}, we use $\nu = 10^{-6}$.

Since the monotonicity probability is not normal, it has to be approximated by a
normal distribution to be used in the GP framework. To understand this, first
consider the problem without the monotonicity constraints, which is given by
Eq.~\eqref{eq:gp_posterior_noised}. The probability distribution of the
observation is given by:
\begin{equation}
  \label{eq:test_latent}
  p(L_* | X_*, X, Y) = \int p(L_* | X_*, X, L) p(L | X, Y) dL,
\end{equation}
where $L$ is the latent variable for the training data, whose probability
distribution, computed by the Bayes' rule, is
\begin{gather*}
  p(L | X, Y) = \frac{p(Y|L) p(L|X)}{p(Y|X)}
  \\
  p(Y|X) = \int p(Y|L) p(L|X) dL.
\end{gather*}
According to the model, the prior $p(L|X)$ and the likelihoods $p(Y|L)$ and
$p(L_*|X_*,X,L)$ are normal distributions, which makes all integrals tractable
and all other distributions defined in the closed form presented in
Eq.~\ref{eq:gp_posterior_noised}.

Now, considering the monotonicity constraints, let $\M$ be the monotonicity
constraints and $L'$ be the random variable associated with the derivative of
the latent variables $L$. Then the probability distribution in
Eq.~\eqref{eq:monotonicity} can be written as $p(\M|L')$. Rewriting the
posterior distribution over the latent variables, we get:
\begin{subequations}
\label{eq:train_latent}
\begin{gather}
  p(L | X, Y, \M) = \frac{p(\M|L') p(Y|L) p(L,L'|X)}{p(Y,\M | X)}
  \\
  \label{eq:train_latent_int}
  p(Y,\M | X) = \int p(\M|L') p(Y|L) p(L,L'|X) dL dL'.
\end{gather}
\end{subequations}

Because the distribution $p(\M|L')$ is not normal and every other
distribution in Eq.~\eqref{eq:train_latent} is normal, the integrals defined in
Eqs.~\eqref{eq:test_latent} and \eqref{eq:train_latent_int} are intractable.
Therefore, the distribution $p(\M|L')$ must be approximated by a normal
distribution, which can be achieved using the expectation propagation
algorithm~\cite{minka2001expectation}, with the update equations described
in~\cite{riihimaki2010gaussian}. The expectation propagation algorithm
iteratively adjusts an unnormalized normal distribution to locally approximate
the distribution defined by the soft constraints, such that $p(\M|L') \approx
\tilde Z \mathcal N(L' | \tilde \mu, \tilde \Sigma)$, where $\tilde Z$ is a
normalization constant, $\tilde \mu$ is a mean vector with one value for each
monotonicity constraint, and $\tilde \Sigma$ is a diagonal covariance matrix.

Besides this monotonicity constraint, we also would like that the errors between
the provided values for the points $z$ and their latent values $l$ are small, so
that the estimated shape of the Pareto frontier is closer to the true one. This
can be achieved by placing a prior inverse-gamma distribution over $\sigma^2$,
whose density is given by:
\begin{equation*}
  p(x; \alpha, \beta) = \frac{\beta^\alpha}{\Gamma(\alpha)} x^{-\alpha-1} \exp
  \left(-\frac{\beta}{x}\right),
\end{equation*}
where $\Gamma(\cdot)$ is the gamma function. As $\beta \to \infty$, this prior
is ignored, while $\beta \to 0$ indicates that there is no noise. In the results
shown, we fix $\alpha = 3$ and vary $\beta$.

We define $f(y)$ as the final expected value $E[l^*|y^*, Z, Y, \theta]$, and the
parameters $\theta$ are optimized to maximize the full likelihood, including
gradient probability and $\sigma^2$ prior, of the training data $Y$ and $Z$. We
also add the monotonicity constraint on all training data for all directions,
but it should be noted that we can also add only monotonicity constraint at a
point without defining its desired value. This allows us to find points that
have $f(y) = 0$ but negative gradient and add the constraint on them, which in
turn could improve the estimation.

To test the GP's performance as a surrogate, we consider the two test frontiers
whose samples are given by $P_1 = [(0, 1), (\epsilon, \epsilon), (1,0)]$, which
is a convex frontier, and $P_2 = [(0, 1), (1-\epsilon, 1-\epsilon), (1,0)]$,
which is a concave frontier, both with $\epsilon = 10^{-3}$. Note that the
points were purposely selected to test the ability to model very sharp
frontiers. However, using only the points defined by $P_1$ and $P_2$ leads to a
solution where $f(y)$ is almost $0$ everywhere. To avoid this problem, we add a
point $(1,1)$, with target value $1$, to $P_1$ and a point $(0,0)$, with target
value $-1$, to $P_2$. The parameters for the Gaussian process are found using
gradient ascent in the samples likelihood.

Figure~\ref{fig:gp_derivative} shows the resulting curves for different values
of $\beta$. The first thing we notice is that, although $\beta \to \infty$ does
not place any restriction on $\sigma$, which allows the observed points in the
frontier to be far from their latent values that actually define the frontier,
the resulting curve is still able to fit the general shape defined by the points
provided.

As we reduce the value of $\beta$, the observed variance $\sigma^2$ is required
to be smaller and the frontier shape gets better and better. Ideally, with
$\beta=0$, the latent points would be the same as the observed points, but
this causes numeric problems due to the monotonicity information and can make it
harder to satisfy the monotonicity constraint, due to the smoothness of the GP.

When we reduce the value to $\beta = 0.01$ and beyond, the resulting frontier
is not valid anymore, with noticeable points with negative derivative. However,
the largest difference in the concave problem is between points $(0.82, 1.055)$
and $(0.2, 0.985)$, with a total reduction in $y_2$ of just $0.07$, and a
similar result is obtained for the convex case. Therefore, this approximation is
still close to the correct frontier and could be used to evaluate proposed
solutions because it was built with the theoretical developments of this paper
in mind and tries to approximate them, which most likely provides better
frontier estimates than methods that use traditional regression solutions, such
as~\cite{moo_svm,loshchilov2010mono,loshchilov2010dominance}, where the manifold
$f(y) = 0$ can have any shape.

\begin{figure*}[t]
  \centering
  \begin{subfigure}[b]{0.4\linewidth}
    \psfrag{f1}[t][t]{\small $f_1$}
    \psfrag{f2}[c][c]{\small $f_2$}
    \includegraphics[width=\linewidth]{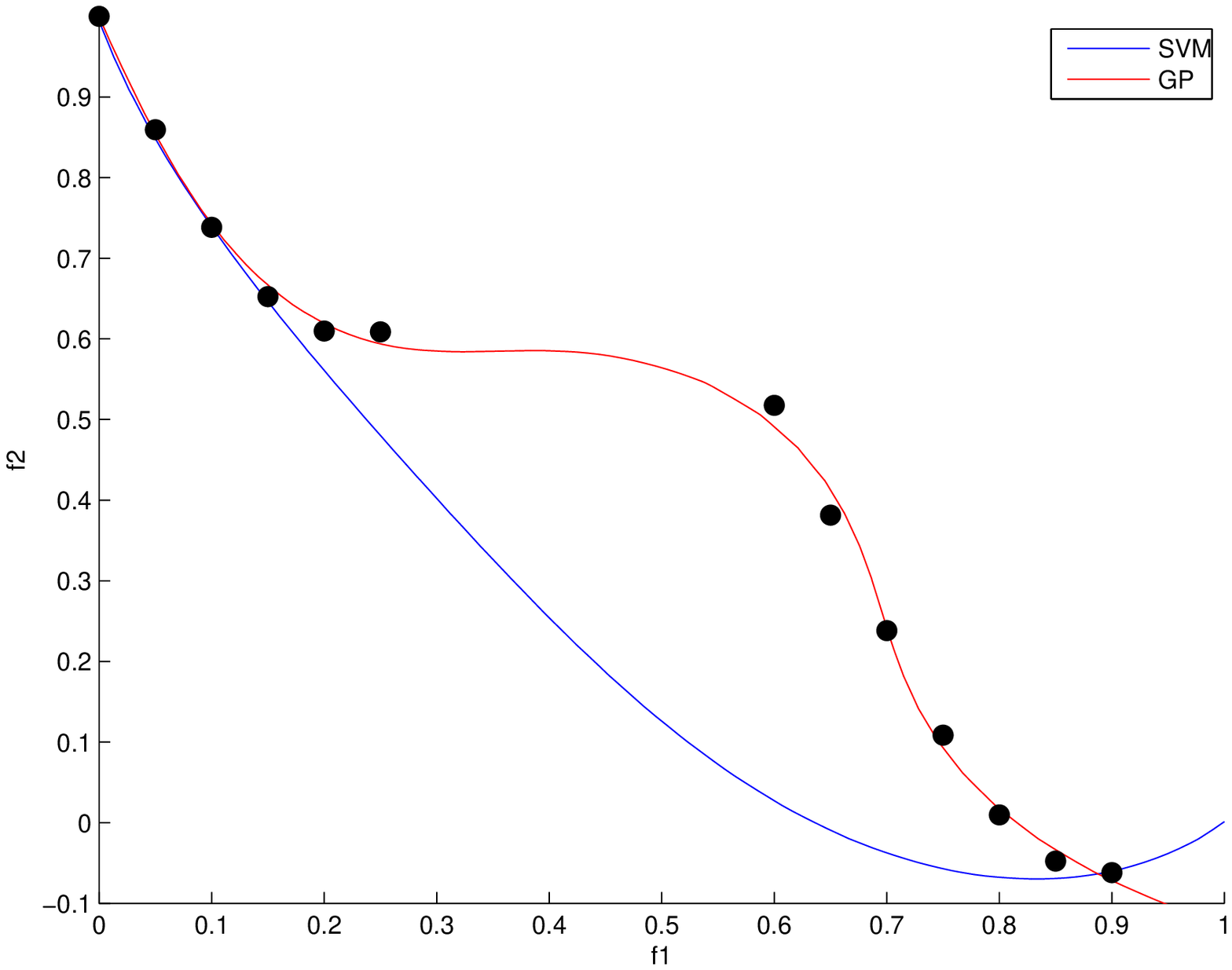}
    \caption{$\gamma = 1$}
  \end{subfigure}
  ~
  \begin{subfigure}[b]{0.4\linewidth}
    \psfrag{f1}[t][t]{\small $f_1$}
    \psfrag{f2}[c][c]{\small $f_2$}
    \includegraphics[width=\linewidth]{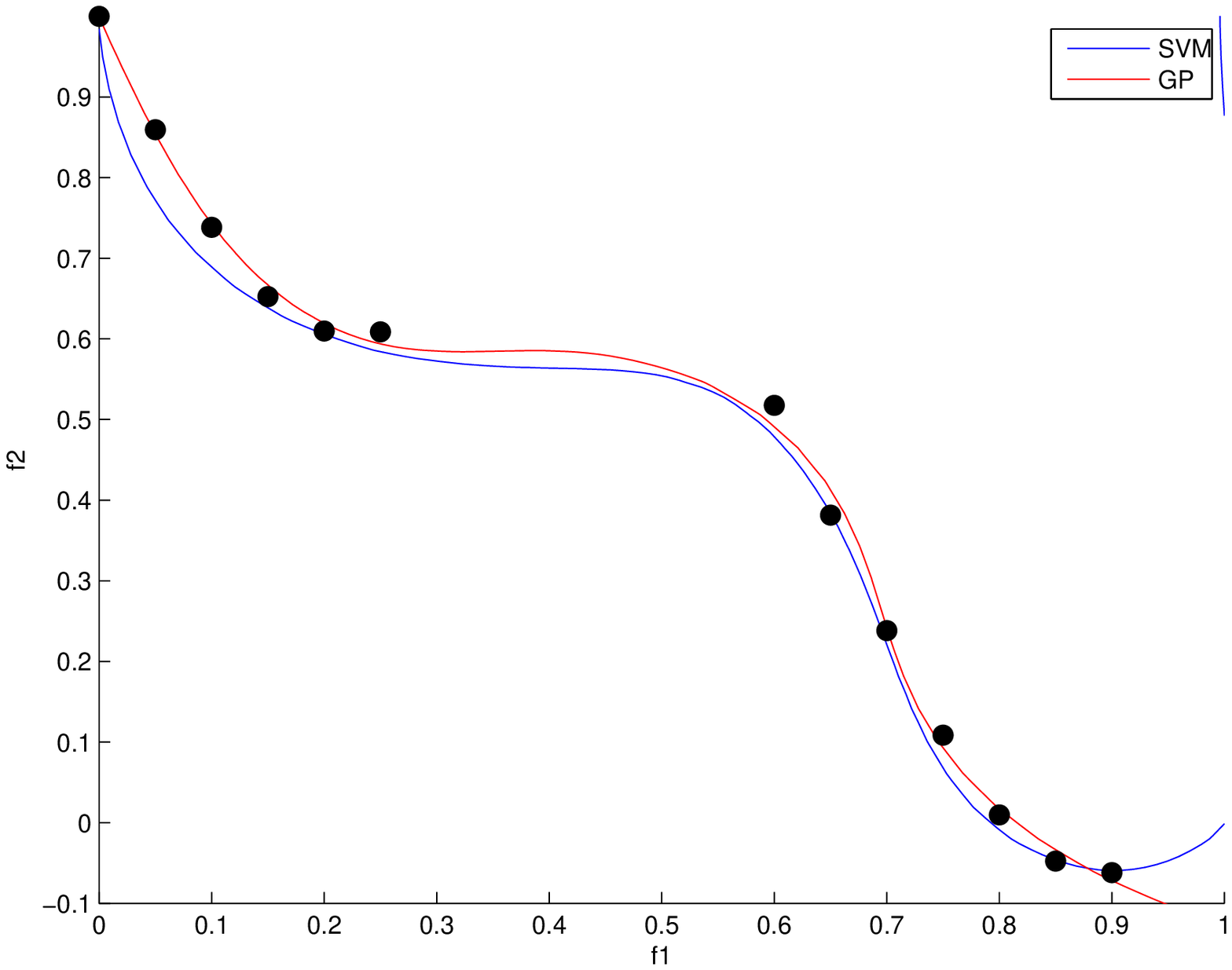}
    \caption{$\gamma = 5$}
    \label{fig:svm_comparison:gamma5}
  \end{subfigure}
  \\
  \begin{subfigure}[b]{0.4\linewidth}
    \psfrag{f1}[t][t]{\small $f_1$}
    \psfrag{f2}[c][c]{\small $f_2$}
    \includegraphics[width=\linewidth]{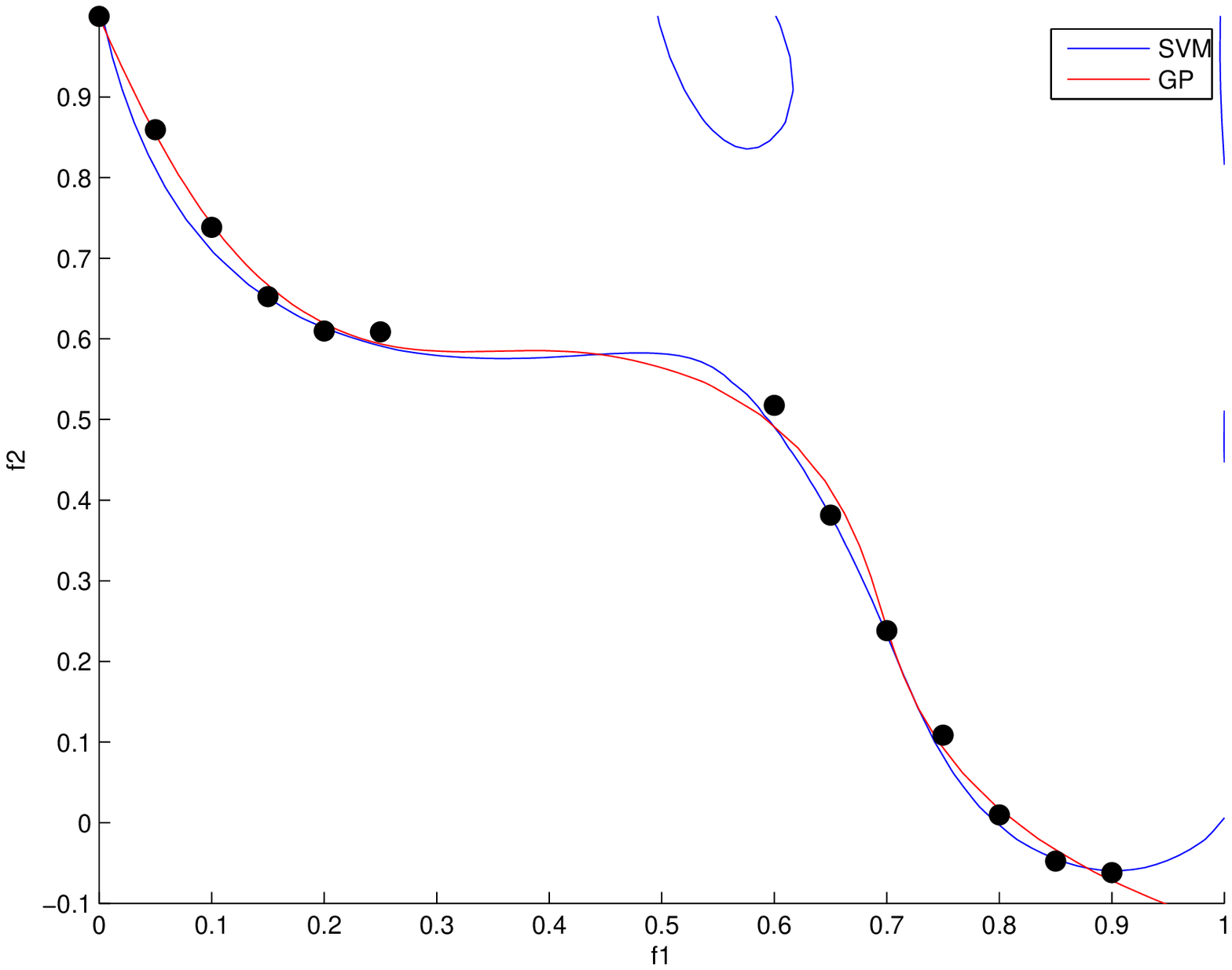}
    \caption{$\gamma = 6$}
    \label{fig:svm_comparison:gamma6}
  \end{subfigure}
  ~
  \begin{subfigure}[b]{0.4\linewidth}
    \psfrag{f1}[t][t]{\small $f_1$}
    \psfrag{f2}[c][c]{\small $f_2$}
    \includegraphics[width=\linewidth]{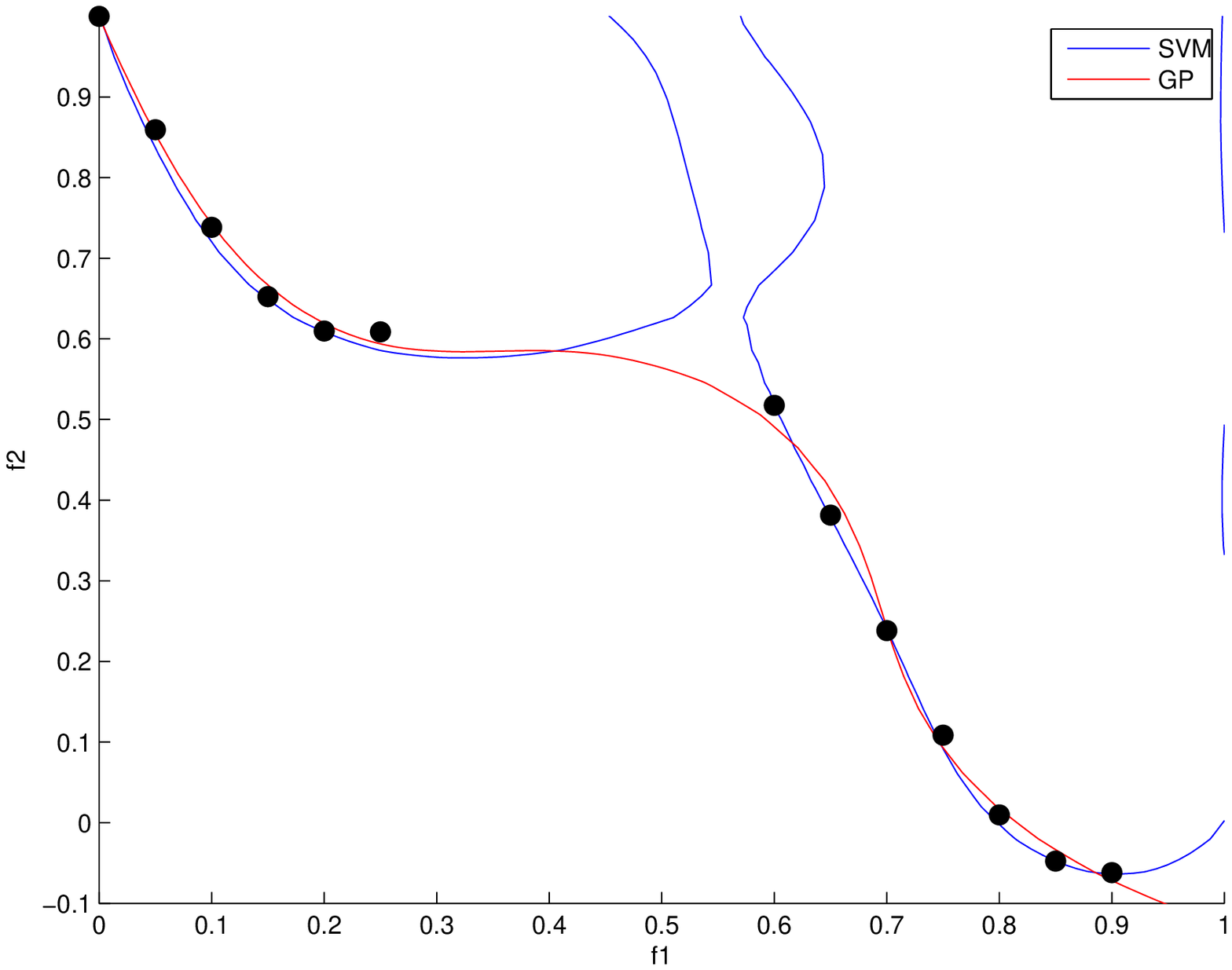}
    \caption{$\gamma = 8$}
  \end{subfigure}
  \caption{Estimated frontiers using SVM with different values of $\gamma$ and
  Gaussian process. The points in the data set that belong to the true Pareto
  frontier are shown as dots.}
  \label{fig:svm_comparison}
\end{figure*}

To evaluate the effect of using the gradient constraint, Fig.~\ref{fig:gp} shows
a similar GP but without any information on the gradients. Although the expected
Pareto frontier is correctly identified, there are also many points that do not
belong to the frontier and where $f(y) = 0$. Since the unconstrained GP had
better frontier estimates for the extreme points than the constrained GP, as all
points between them and the knee satisfy the conditions, it appears that
not every point benefits from the gradient constraint.

Even though both GP models failed to fully satisfy the theoretical conditions,
we consider that the GP with derivative restriction performed better, both
because there are some parameter sets that are able to satisfy the frontier
conditions and because it does not violate the restrictions as much. Moreover,
if the variance, which is not shown but is higher for points far from the inputs
provided, is taken into account, then the violations of the GP with derivatives
occur in a region with higher uncertainty than the violations of the pure GP.

Therefore, despite the minor violations of the GP with derivative constraints,
this approximation is still close to the correct frontier and could be used to
evaluate the proposed solutions.
\subsection{Comparison to Existing SVM Surrogate}
The surrogate method introduced in \cite{moo_svm}, like the method proposed in
this paper, is based on approximating the frontier directly from values in the
objective space. This makes it a good candidate for comparison and validating
the conjecture that existing methods may arbitrarily violate the conditions
described in this paper.

The one-class SVM used in \cite{moo_svm} is defined by the following
optimization problem:
\begin{align*}
  \min_{w, \xi_i, \rho} ~& \frac{\|w\|^2}{2} + \frac{1}{\nu N} \sum_{i=1}^N
  \xi_i - \rho
  \\
  \text{s.t.} ~& w^T \phi(x_i) \ge \rho - \xi_i
  \\
  & \xi_i \ge 0, i \in \{1,\ldots,N\},
\end{align*}
where $v \in (0,1]$ and the feature-extraction function $\phi(x)$ is defined
implicitly by the kernel
\begin{equation*}
  K(x,y) = \exp(-\gamma \|x-y\|^2),
\end{equation*}
which is similar to the kernel used for the GP.

One important difference between training an SVM and a GP is that the GP has a
natural way to optimize its hyper-parameters by maximizing the data likelihood,
which automatically defines a trade-off between fitting the data and model
complexity. For the SVM, we must use cross-validation \cite{bishop2006pattern},
which reduces the number of points available to fit the model, since the data
must be divided in the training and validation sets.

To compare the surrogate methods, we use one test problem from
\cite{deb2001multi}, which is also used in \cite{moo_svm} to show the behavior
of the proposed SVM surrogate. The problem is given by:
\begin{align*}
  \min ~& f_1(x_1,x_2) = x_1
  \\
  \min ~& f_2(x_1,x_2) = 1+x_2^2 - x1 - 0.2 \sin(3 \pi x_1)
  \\
  \text{s.t.} ~&
  x_1 \in [0,1],~ x_2 \in [-2,2].
\end{align*}
We chose this problem because its true Pareto frontier is discontinuous, which
creates sharp changes in its associated estimated Pareto frontier, just like in
Fig.~\ref{fig:pareto_frontier_example}, and makes it harder to approximate.

We chose $\nu = 10^{-3}$ so that the samples provided should be almost perfectly
classified and we constraint the scales $\rho_i$ in Eq.~\ref{eq:rbf_kernel} to
be equal, so that both methods can use the same features from the samples. The
data set provided is composed of a grid with step $0.05$ for both variables,
which includes some points in the Pareto frontier. The full grid is used to fit
the SVM because it provided better results than using just the non-dominated
points, while only the non-dominated points are required for the GP.

Figure~\ref{fig:svm_comparison} shows the resulting approximations of the Pareto
frontier using a GP with parameters learnt through gradient ascent in the data
likelihood, like in Section~\ref{sec:gp_monotonic}, and an SVM with different
values of $\gamma$. The GP learns an appropriate shape from the samples provided
despite the discontinuity in the frontier, but also slightly violates the
constraints during the gap in $f_1 \in [0.3, 0.5]$. Moreover, in the absence of
any information about the shape in the interval $f_1 \in (0.9,1]$, because no
point was provided there, the GP extrapolates a valid shape for the Pareto
frontier.

The SVM is highly dependent on the parameter $\gamma$. When it is small, the
shape learnt is very conservative and does not follow the shape defined by the
points in the frontier. On the other hand, when it is large, the surrogate fits
the points in the frontier better but also may define a function that violates
greatly the conditions to be a valid Pareto frontier. The best value for
$\gamma$ that does not violate the constraints in the interval $f_1 \in [0,
0.9]$ is $\gamma = 5$. However, for this value the GP provides a better
approximation of the Pareto frontier, as shown in
Fig.~\ref{fig:svm_comparison:gamma5}. Increasing $\gamma$ provides a better
approximation, achieving a quality comparable to the GP, but also creates
regions that violate the conditions to be a valid Pareto frontier more than the
GP. Furthermore, $\gamma = 6$ defines a region that the SVM believes is part of
the Pareto frontier but actually is very distant from it and inside the
dominated region, as shown in Fig.~\ref{fig:svm_comparison:gamma6}.

Besides these issues, the SVM also does not extrapolate well to the region $f_1
\in (0.9,1]$. Close inspection shows that the dominated region defined by the
SVM is finite, that is, it is described by a region in the objective space that
is surrounded by an infinite region that the SVM believes is not dominated. This
behavior shows that the learnt model carries no concept of the problem it is
solving, which is to approximate a Pareto frontier, but describes a generic
function approximation. The results in Fig.~\ref{fig:svm_comparison} provide
evidence for the conjecture that existing methods proposed in the literature may
arbitrarily violate the conditions described in this paper.

Furthermore, if only the points at the Pareto frontier were provided for
learning, then the region defined by the SVM would enclose only these points and
would ignore the dominated region. Thus the SVM method requires data in the
dominated region while the GP method only requires the points at the frontier.

\section{Conclusion}
\label{sec:conclusion}
In this paper, we have introduced the necessary and sufficient conditions that
functions must satisfy so that their solution space describes an estimated
Pareto frontier. These conditions follow from the definition of an estimated
Pareto frontier and are extended for differentiable functions, which allows
easier verification of the conditions.

Based on these conditions, a Gaussian process (GP) was tested on toy problems
with very sharp Pareto frontiers. The GP was extended to include the theoretical
conditions as soft probabilistic constraints and a regularization term was added
to avoid large deviations between the points and their latent values. The mean
latent value is used as surrogate for the Pareto frontier, and some values of
the regularization constant allowed a correct frontier estimate to be found.

However, when the regularization becomes too strong, the surrogate violates the
constraints that define a valid estimated Pareto frontier on some points, but
this occurs far from the given inputs and the deviation is small. This suggests
that, even under these conditions, the proposed function could be used to
provide insight on the shape of the true Pareto frontier, and possibly provide
more realistic estimates than other methods that do not take the restrictions
into consideration during their design.

To validate this hypothesis and the conjecture that existing surrogate methods
may violate the conditions described in this paper, we compared the proposed GP
with a one-class SVM used in~\cite{moo_svm} on one of the test problems
described in the same paper. We showed that the GP again violates the
constraints by small values and provide a good estimate for the Pareto frontier,
while the SVM defined a worse estimate or violated the conditions more than the
GP. Furthermore, the dominated region defined by the SVM is bounded by what it
represents as the non-dominated region, while the GP correctly divides the space
in two infinite areas.

Besides being a better surrogate for the Pareto frontier, the GP has the data
likelihood as an innate measure that can be used to optimize its
hyper-parameters and only requires data at the frontier. On the other hand, the
SVM must use some method, like cross-validation~\cite{bishop2006pattern}, to
optimize its hyper-parameters and it requires data in the dominated region to
define better approximations.

We highlight that, although GP were used together with the theory on this paper
to approximate the Pareto frontier, the theory is general and does not depend on
the specific choice of the function descriptor. Therefore, other models that are
able to deal with the constraints imposed by the theory, in either a soft or
hard way, should be able to learn the desired shape of the Pareto frontier too.
Nonetheless, we are not aware of any other method to create the score function
in which the constraints are as easy to include as in the GP. Additionally, a GP
provides robustness to changing the number of points used in the estimation.

Further investigations involve studying the behavior of the GP to approximate
the Pareto frontier with real benchmarks and using some multi-objective
optimization algorithm, such as NSGA-II~\cite{deb2002fast}, to provide the
points. Since the objectives tend to be smoother than in the example frontier
provided~\cite{huband2006review}, we expect the estimated Pareto frontier
described by a GP to fit the true Pareto frontier even better in these problems.
If this is the case, we will investigate the possibility of integrating the
frontier surrogate with other surrogate models for the objectives, so that all
of them are learned directly and the number of function evaluations could be
reduced.

Moreover, since the only requirement for the surrogate is that the Pareto
frontier is approximated by the null space and the exact value on other parts of
the objective space are not relevant, the GP could be used to fit a regression
model on the individuals of a population where the target value is monotonically
increasing in the objective space. Standard performance measures in
multi-objective optimization, such as the class in non-dominated
sorting~\cite{deb2002fast} and the dominance count~\cite{beume2007sms}, satisfy
this property and can be used as targets of the regression. In this case, the GP
would not only define the Pareto frontier, but would also define a measure of
the distance between a given point and the approximated Pareto frontier.

Another interesting line of research is to evaluate when the derivative
constraints on the points provided is beneficial, since in some points it avoids
incorrect association of other points with the frontier, like around the knee in
the unconstrained GP shown in this paper, and in others it may make the
estimated shape not satisfy the constraints, like the points in the constrained
GP also shown in this paper. This could not only provide better fit, but may
also increase the fitting speed, since less constraints need to be evaluated,
which reduces the size of the GP and the number of expectation propagation steps
required. Therefore an iterative algorithm that adds the constraints as needed
should be pursued.

\section*{Acknowledgment}
The authors would like to thank CNPq for the financial support.

\bibliographystyle{templates/IEEEtran/bib/IEEEtran}
\bibliography{paper}

\vspace*{-2\baselineskip}

\begin{IEEEbiographynophoto}{Conrado S. Miranda}
  received his M.S. degree in Mechanical Engineering and his B.S. in Control
  and Automation Engineering from the University of Campinas (Unicamp), Brazil,
  in 2014 and 2011, respectively. He is currently a Ph.D. student at the School
  of Electrical and Computer Engineering, Unicamp. His main research interests
  are machine learning, multi-objective optimization, neural networks, and
  statistical models.
\end{IEEEbiographynophoto}

\begin{IEEEbiography}[{\includegraphics[width=1in,height=1.25in,clip,keepaspectratio]{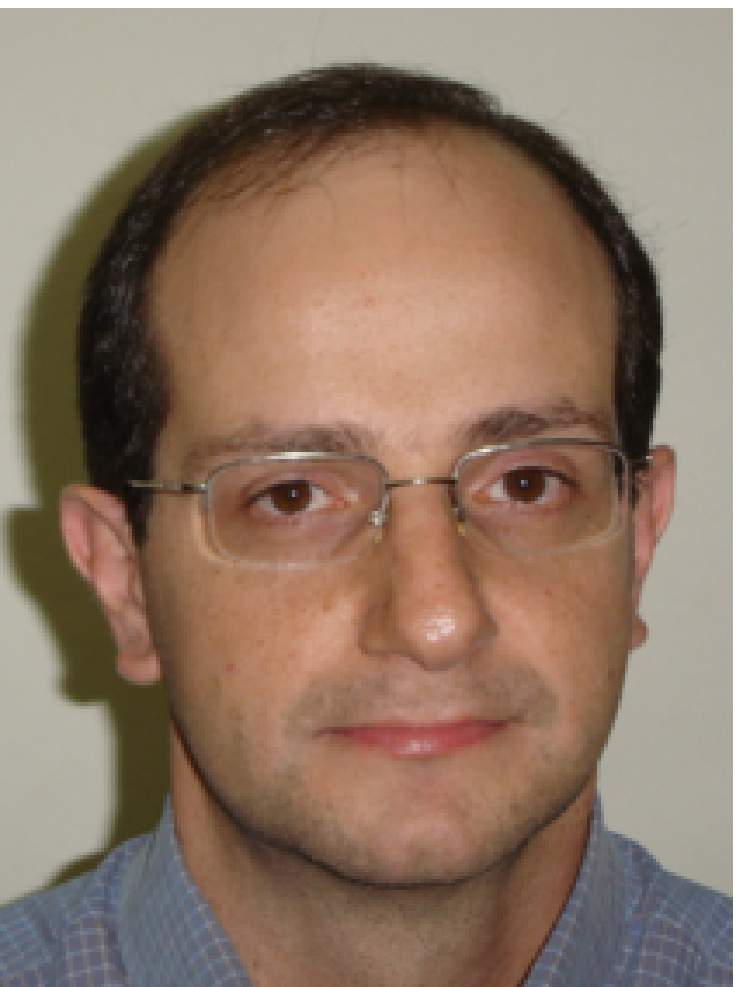}}]{Fernando J. Von Zuben}
  received his Dr.E.E. degree from the University of Campinas (Unicamp),
  Campinas, SP, Brazil, in 1996. He is currently the header of the Laboratory of
  Bioinformatics and Bioinspired Computing (LBiC), and a Full Professor at the
  Department of Computer Engineering and Industrial Automation, School of
  Electrical and Computer Engineering, University of Campinas (Unicamp). The
  main topics of his research are computational intelligence, natural computing,
  multivariate data analysis, and machine learning. He coordinates open-ended
  research projects in these topics, tackling real-world problems in the areas
  of information technology, decision-making, pattern recognition, and discrete
  and continuous optimization.
\end{IEEEbiography}

\vfill

\end{document}